\documentclass{article} % For LaTeX2e
\usepackage{amsfonts,amsmath,amssymb,amsthm,url,xspace}
\usepackage{graphicx,subfig}
\usepackage{fullpage}
\usepackage{natbib}
\usepackage{authblk}
\usepackage[pointedenum]{paralist}
\usepackage[ruled,noline]{algorithm2e}
\usepackage{epic,eepic,eepicemu}
\usepackage{epsfig}
\usepackage{fancyhdr}
\usepackage{graphics}
\usepackage{psfrag,latexsym}
\usepackage{multirow}
\usepackage{tabularx}
\usepackage{epsfig}
\theoremstyle{break}
\newtheorem{theorem}{Theorem}
\newtheorem{lemma}{Lemma}

\usepackage{enumitem}

\newcommand{\R}{{\mathbb{R}}}

\newcommand{\V}{{\mathcal{V}}}

\newcommand{\bc}{{\boldsymbol{c}}}
\newcommand{\by}{{\boldsymbol{y}}}
\newcommand{\be}{{\boldsymbol{e}}}

\newcommand{\bx}{{\boldsymbol{x}}}
\newcommand{\AL}{\boldsymbol{\alpha}}

\def\covtype{{\sf covtype}\xspace}
\def\liblinear{{\sf LIBLINEAR}\xspace }
\def\LIBSVM{{ LIBSVM}\xspace}
\def\SVMLIGHT{{ SVMLight}\xspace}
\def\MNIST{{\sf mnist8m}\xspace}
\def\ijcnn{{\sf ijcnn1}\xspace}
\def\covtype{{\sf covtype}\xspace}
\def\webspam{{\sf webspam}\xspace}
\def\census{{\sf census}\xspace}
\def\ijcnn{{\sf ijcnn1}\xspace}
\def\kddcup{{\sf kddcup99}\xspace}
\def\cifar{{\sf cifar}\xspace}
\def\a9a{{\sf a9a}\xspace }

\title{A Divide-and-Conquer Solver for Kernel Support Vector Machines}
\author[1]{Cho-Jui Hsieh}
\author[1]{Si Si} 
\author[1]{Inderjit S. Dhillon}
\affil[1]{Department of Computer Science, University of Texas, Austin}

\date{}

\begin{document}

%\twocolumn[
%\icmltitle{ A Divide-and-Conquer Solver for Kernel Support Vector Machines}
% Support Vector Machines}
% unless you've provided the [accepted] option to the icml2013
%\icmlauthor{Your Name}{email@yourdomain.edu}
%\icmladdress{Your Fantastic Institute,
%            314159 Pi St., Palo Alto, CA 94306 USA}
%\icmlauthor{Your CoAuthor's Name}{email@coauthordomain.edu}
%\icmladdress{Their Fantastic Institute,
%            27182 Exp St., Toronto, ON M6H 2T1 CANADA}
%
%% You may provide any keywords that you 
%% find helpful for describing your paper; these are used to populate 
%% the "keywords" metadata in the PDF but will not be shown in the document
%\icmlkeywords{boring formatting information, machine learning, ICML}
%\vskip 0.3in
%]
\maketitle

\begin{abstract}
The kernel support vector machine (SVM) is one of the most widely used 
classification methods; however, the amount of computation required 
becomes the bottleneck when facing millions of samples. 
In this paper, we propose and analyze a novel divide-and-conquer 
solver for kernel SVMs (DC-SVM). In the division step, we partition the kernel SVM problem into smaller subproblems by clustering the data, so that
each subproblem can be solved independently and efficiently. 
We show theoretically that the support vectors identified by the subproblem 
solution are likely to be support vectors of the entire kernel SVM problem, provided
that the problem is partitioned appropriately by kernel clustering. 
%will be close to the whole problems
%our method minimize the difference between solutions of subproblems and of the whole problem. 
%Furthermore, the set of support vectors identified by subproblems is
%similar to the support vector of the whole problem. 
%By analyzing the difference of solutions between subproblems and the whole problem, 
%we show that kernel kmeans 
%, where 
%each of them can be solved individually and efficiently.
In the conquer step, the local solutions from the subproblems are used to initialize a 
global coordinate descent solver, which converges quickly as suggested by our analysis. 
By extending this idea, we develop a multilevel Divide-and-Conquer SVM
algorithm with adaptive clustering and early prediction strategy, 
which 
%We show that running DC-SVM with multiple levels
%with adaptive clustering and early prediction strategy, 
%and meanwhile we use
%adaptive clustering approach to refine the partition until we arrive at the 
%original problem. 
outperforms state-of-the-art methods in terms of 
training speed, testing accuracy, and memory usage. 
As an example,
on the \covtype dataset with half-a-million samples,
DC-SVM is 7 times faster than LIBSVM in obtaining  
the exact SVM solution (to within $10^{-6}$ relative error) 
which achieves 96.15\% prediction accuracy. 
%a solution with $10^{-6}$ 
%relative error. 
Moreover, with our proposed early prediction strategy, 
DC-SVM achieves about 96\% accuracy 
in only 12 minutes, 
which is more than 100 times faster than LIBSVM. 
%while other 
%state-of-the-art methods need more than 10 hours to 
%achieve such performance.

\end{abstract}
\section{Introduction}
The support vector machine (SVM) \citep{CC95a} is probably the most widely used classifier in varied machine learning applications. 
%For linearly separable problems, linear SVMs, which work in the input space, 
%can achieve good accuracy and can be solved efficiently \citep{CJH08a, SSS11a}. 
%instances from different classes, can be solved efficiently \citep{CJH08a,SSS11a} but 
%has poor performance on some non-linearly separable datasets. 
For problems that are not linearly separable, kernel SVM uses a ``kernel 
trick'' to implicitly map samples from input space to a high-dimensional feature 
%, sometimes infinite dimensional feature 
space, where samples become linearly separable. 
Due to its importance, optimization methods for kernel SVM have been widely studied 
\citep{JP98a,TJ98a}, and efficient libraries such as \LIBSVM \citep{CC01a} and 
\SVMLIGHT \citep{TJ98a}
%which employ state-of-the-art solutions to handle large 
%sample sizes 
are well developed.
However, the kernel SVM is still hard to scale up when the sample size reaches more 
than one million instances. The 
bottleneck stems from the 
high computational cost and memory requirements of computing and storing the kernel matrix, 
which in general is not sparse. 
%Many approximate solvers for kernel SVMs have also been proposed \citep{SF01,ZK12a,fastfood}. 
By approximating the kernel SVM objective function, 
approximate solvers \citep{ZK12a,fastfood}
avoid high computational cost and memory requirement, but suffer in terms of 
prediction accuracy. 
%they avoid high computational cost and memory requirement, but they usually 
%suffer in terms of prediction accuracy.  
%competitive testing performance. 
%Since solving the exact kernel SVM problem is time consuming and computational prohibitive for large-scale problems, 
%many approximate solvers have been proposed. For example, squashing SVM \citep{DP00a}, 
%Nystr\"{o}m based kernel approximation for kernel SVM \citep{SF01,CW01a,PD05a,ZK12a}, and 
%greedy support vector selection approaches \citep{simpleSVM,SSK06a}. 
%Although these approximate solvers achieve
%good testing accuracy with fewer training time than the exact solvers, 
%they strongly depend on the approximation to the original problem.  

In this paper, we propose a novel divide and conquer approach (DC-SVM) to efficiently solve 
the kernel SVM problem. DC-SVM achieves 
faster convergence speed compared to state-of-the-art
{\em exact} SVM solvers, as well as better prediction accuracy in much less time
than {\em approximate} solvers.
To accomplish this performance, DC-SVM first divides the full problem into smaller 
subproblems, which can be solved 
independently and 
efficiently.
We theoretically show that the kernel kmeans algorithm is able to minimize
the difference between the solution of subproblems and of the whole problem,  
and support vectors identified by subproblems are very likely to be support vectors 
of the whole problem. 
However, running kernel kmeans on the whole dataset is time consuming, 
so we apply a two-step kernel kmeans procedure to efficiently find the partition. 
%We show that the solution of subproblems is close
%to the solution of the original problem when kernel kemeans 
%the solution difference between the subproblems 
%and original one is 'close' using clustering, 
%which implies the subproblems' solutions 
%would be a good initial point for the whole problem. So 
In the conquer step, the local solutions from the subproblems are ``glued'' together to 
yield an initial point for the 
global problem.
%In the conquer step,  we initialize the solver for the whole 
%problem from a good initial point using the subproblems' solutions. 
As suggested by our analysis, 
%Since support vectors identified from 
%the subproblems are very close to support vectors for the whole problem,
the coordinate descent method in the final stage converges quickly 
to the global optimal. 
%We further propose to run DC-SVM with multiple levels to achieve further speedup.  

%In order to solve kernel SVM problems efficiently, 
%we propose a divide and conquer based approach(DC-SVM). DC-SVM achieves 
%both faster convergence speed comparing with state-of-the-art
%exact SVM solvers, and better prediction accuracy  
%than approximate solvers with fewer training time.
%The key idea of our algorithm is to divide the whole kernel SVM problem into smaller 
%subproblems by clustering, which can be solved individually and efficiently. 
%We show that the approximation error can be bounded by the sum of between-cluster 
%kernel values of support vectors, and therefore we propose an adaptive 
%two-step kernel kmeans algorithm to achieve a good data partition. 

%Using the solutions of subproblems, we can initialize the solver for the whole 
%problem from a good initial point. Moreover, support vectors identified from 
%the subproblems are very close to support vectors for the whole problem, so 
%the coordinate descent method converges very fast to the global optimal for 
%the global problem. We run DC-SVM hierarchically to achieve further speedup.  

Empirically, our proposed Divide-and-Conquer Kernel SVM solver can reduce 
the objective function
value much faster than existing SVM solvers. For example, on the \covtype dataset with half a million
samples, DC-SVM can find an accurate globally optimal solution (to within $10^{-6}$ accuracy) 
within 3 hours on a single machine with 8 GBytes RAM, while the state-of-the-art solver 
LIBSVM
takes more than 22 hours to achieve a similarly accurate solution (which yields 96.15\% prediction accuracy). 
More interestingly, 
due to the closeness of the subproblem solutions to the global solution, 
we can employ an early prediction approach, using which 
%combined with our proposed early prediction approach, 
DC-SVM can obtain high test accuracy extremely quickly. 
For example, on the \covtype dataset, by using early prediction DC-SVM achieves 96.03\% prediction accuracy 
within 12 minutes, while other solvers cannot achieve such performance in 10 hours. 

The rest of the paper is outlined as follows. 
We propose the single-level DC-SVM in Section \ref{sec:dcsvm},  
and extend it to the multilevel version in Section \ref{sec:multilevel}. 
Experimental comparison with other state-of-the-art SVM solvers is shown in 
Section \ref{sec:experiment}. 
The relationship
between DC-SVM and other methods is discussed in Section \ref{sec:related},  
and the conclusions are given in Section \ref{sec:conclusion}. 
Extensive experimental comparisons are included in the Appendix. 
\section{Related Work}
\label{sec:related}
The optimization methods for SVM training has been intensively studied.  
Since training an SVM requires a large amount of memory, it is natural 
to apply decomposition methods \citep{JP98a}, where
only a subset of variables is updated at each step. 
Based on 
this idea, 
software packages such as \LIBSVM \citep{CC01a}  and \SVMLIGHT \citep{TJ98a} are well developed.
To speed up the decomposition method, 
\citep{FP04b} proposed a double chunking approach to maintain a chunk of important samples, 
and the shrinking technique \citep{TJ98a} is also widely used to eliminate unimportant samples. 
%: they use a large chunk to store important
%samples, and run a decomposition method inside it.

To speed up kernel SVM training on large-scale datasets, 
it is natural to divide the problem into smaller subproblems, 
and combine the models trained on each partition. 
\citep{RAJ91a} proposed a way to combine models, although in their algorithm subproblems
are not trained independently, while 
\citep{Tresp} discussed a Bayesian prediction scheme (BCM) for model combination.  
\citep{RC02a} partition the training dataset arbitrarily in the beginning, and then iteratively refine
the partition to obtain an approximate kernel SVM solution. \citep{MK06a} applied the above ideas to solve multi-class problems. 
\citep{cascadeSVM} proposed a multilevel approach (CascadeSVM): 
 they randomly build a partition tree of samples
%Another related approach is Cascade SVM \citep{cascadeSVM}. They build a partition tree of the dataset
and train the SVM in a ``cascade'' way: only support vectors in the lower level of the tree
are passed to the upper level.
%\citep{MK06a} further applied the above idea to solve multi-class problems.  
However, no earlier method appears to discuss an elegant way to partition the data.  
In this paper, we theoretically show that kernel kmeans minimizes the error of 
the solution from the subproblems and the global solution.  
%and we propose a two-step kmeans with adaptive sampling to speedup the data division step. 
Based on this division step, we propose a simple method to combine locally trained SVM models, 
and show that the testing performance is better than BCM in terms of both accuracy and time (as presented in Table \ref{tab:prediction_method}). More importantly, DC-SVM solves the original SVM problem, not just an approximated one. We compare our method with Cascade SVM in the experiments. 

Another line of research proposes to reduce the training time by 
representing the whole dataset using a smaller set of landmark points,
and clustering is an effective way to find landmark points (cluster centers). 
\citep{MJ89a} proposed this idea to train the reduced sized problem with RBF kernel (LTPU); 
%and then represent other samples by the inner product to landmark points instead of the 
%whole dataset. 
%Earlier work such as \citep{MJ89a} applied clustering to find landmark points,  
%and use the RBF kernel on these landmark points to perform the classification. 
 \citep{DP00a} used a similar idea as a preprocessing of the dataset, while \citep{HY05a} further generalized this approach to a hierarchical 
 coarsen-refinement solver for SVM. 
 Based on this idea, the kmeans Nystr\"{o}m
 method \citep{KZ08a} was proposed to approximate the kernel matrix using 
 landmark points. 
\citep{DB04a} proposed to find samples with similar $\AL$ values by clustering, 
so both the clustering goal and training step are quite different from ours. 
%where each cluster is represented by a point so all the within-cluster $\AL$'s have the same values, so both the clustering goal and training step is quite different. 
% and the resulting SVM solver was proposed in \citep{ZK12a}.
All the above approaches focus on modeling the between-cluster (between-landmark points) relationships.  
In comparison, our method emphasizes on preserving the within-cluster relationships
at the lower levels and explores the 
%We capture the within-cluster relations at the lower level and explore the 
between-cluster information in the upper levels. 
%Moreover, 
%our method DC-SVM computes the exact kernel SVM solution. 
We compare DC-SVM with LLSVM (using kmeans Nystr\"{o}m) and 
LTPU in Section \ref{sec:experiment}. 

There are many other approximate solvers for the kernel SVM, 
including kernel approximation approaches \citep{SF01, ZK12a, fastfood}, 
greedy basis selection \citep{SSK06a}, and online SVM solvers \citep{AB05a}. 
%Finally we discuss the kernel approximation approach \citep{CW01a, SF01}, 
%which aims to approximate the kernel function and solve the approximate problem. 
%Recently research \citep{ZK12a, fastfood} showed that one can efficiently solve
%large-scale problems by approximating the kernel using Nystr\"{o}m approximation or random Fourier features. 
%However, these approaches result in  
%substantial decrease in test accuracy. 
Recently, 
\citep{DeepSVM} proposed an approximate solver to reduce testing time, 
while our work is focused  on reducing the training time of kernel SVM.

\section{Divide and Conquer Kernel SVM with a single level}
  \label{sec:dcsvm}
%  In this paper, we focus on solving  
%  the kernel SVM problem without the bias term. 
  Given a set of instance-label pairs 
  $(\bx_i, y_i), i=1, \dots, n, \bx_i\in \R^d$ and $y_i \in \{1,-1\}$, 
  the main task in training the kernel SVM is to solve the following quadratic 
  optimization problem: 
\begin{equation}
  \min_{\AL}   f(\AL) = \frac{1}{2}\AL^T Q \AL - \be^T \AL, \ \text{ s.t. }
 0\leq \AL \leq C,  
   \label{eq:dual_nobias}
 \end{equation}
 where $\be$ is the vector of all ones; $C$ is the balancing parameter between loss 
 and regularization in the SVM primal problem; 
 $\AL\in \R^n$ is the vector of dual variables; and $Q$ is an $n\times n$ matrix 
 with $Q_{ij}=y_i y_j K(\bx_i, \bx_j)$, where $K(\bx_i,\bx_j)$ is the kernel 
 function.
 Note that, as in \citep{SSK06a,TJ06a}, we ignore the ``bias'' term -- 
 indeed, in our experiments reported in Section \ref{sec:experiment}, 
 we did not observe any improvement in test accuracy by including the bias term. 
 Letting $\AL^*$ denote the optimal solution of \eqref{eq:dual_nobias}, 
 the decision value for a test data $\bx$ can be  computed by 
 %\begin{equation}
 %  \label{eq:svm_decisionvalue}
  $ \sum_{i=1}^n \alpha_i^* y_i  K(\bx,\bx_i)$. 
% \end{equation}

 We begin by describing the single-level version of our proposed algorithm. 
 The main idea behind our divide and conquer SVM solver (DC-SVM) is to divide the data
 into smaller subsets, where each subset can be handled efficiently and independently. 
 The subproblem solutions are 
 %solution computed by the subsets, denoted by $\bar{\AL}$, is 
 then used to initialize
 a coordinate descent solver for the whole problem. 
% conquer them in an good manner such that to 
% solve the original problem \eqref{eq:dual_nobias} efficiently.
 To do this, we first partition the dual variables
 into $k$ subsets $\{\V_1, \dots, \V_k\}$, 
 and then solve the respective subproblems independently 
 %for $c=1,\dots,k$:
 \begin{equation}
   \min_{\AL_{(c)}} \!  \frac{1}{2}(\AL_{(c)})^T Q_{(c,c)} \AL_{(c)} \! -\! \be^T \AL_{(c)}, 
    \text{ s.t. }
   0\!\leq\! \AL_{(c)}\! \leq\! C, 
   \label{eq:dual_subpb}
 \end{equation}
 where $c=1,\dots,k$,  
 $\AL_{(c)}$ denotes the subvector $\{\alpha_p\mid p\in \V_c\}$
 and $Q_{(c,c)}$ is the submatrix of $Q$ with row and column indexes $\V_c$. 
 
 The quadratic programming problem \eqref{eq:dual_nobias} has $n$ variables, 
 and generally takes $O(n^3)$ time to solve. By dividing it into $k$ subproblems 
 \eqref{eq:dual_subpb} with 
 equal sizes, 
 the time complexity for solving the subproblems can be 
 dramatically reduced to $O(k\cdot (\frac{n}{k})^3)=O(n^3/k^2)$. 
 Moreover, the space requirement is also reduced from $O(n^2)$ 
 to $O(n^2/k^2)$. 
% which is much lower than solving the original problem.
 
 After computing all the subproblem solutions, we concatenate them to form 
 an approximate solution for the whole problem
 $\bar{\AL} = [\bar{\AL}_{(1)},\dots,\bar{\AL}_{(k)}]$, 
 where $\bar{\AL}_{(c)}$ is the optimal solution for 
 the $c$-th subproblem. In the conquer step,
 $\bar{\AL}$ is used to initialize the solver for the whole problem. We show that 
 this procedure achieves faster convergence due to the following reasons: (1) $\bar{\AL}$ is close to the optimal solution for the whole 
 problem ${\AL}^*$, so the solver only requires a few iterations to converge
 (see Theorem \ref{thm:approx}); 
 (2) 
% \ref{thm:support_vector}, 
 the set of support vectors of the subproblems is expected to be close to the set of  
 support vectors of the whole problem (see Theorem \ref{thm:support_vector}). Hence, the coordinate descent solver for the whole problem
 converges very quickly. 

 {\bf Divide Step. }
 We now discuss in detail how 
 to divide problem \eqref{eq:dual_nobias} into 
 subproblems.  In order for our proposed method to be efficient, we require
 $\bar{\AL}$ to be close to the optimal solution of the original problem $\AL^*$. 
 In the following, we derive a bound on $\|\bar{\AL}-\AL^*\|_2$ by 
 first showing that $\bar{\AL}$
 is the optimal solution of \eqref{eq:dual_nobias} with an approximate kernel. 
 \begin{lemma}
   \label{thm:tildek}
   $\bar{\AL}$ is the optimal solution of \eqref{eq:dual_nobias}
   with kernel function $K(\bx_i,\bx_j)$ replaced by 
 \begin{equation}
   \bar{K}(\bx_i, \bx_j) = I(\pi(\bx_i), \pi(\bx_j)) K(\bx_i, \bx_j),  
   \label{eq:tildek}
 \end{equation}
 where $\pi(\bx_i)$ is the cluster that $\bx_i$
 belongs to; $I(a,b)=1$ iff $a=b$ and $I(a,b)=0$ otherwise. 
 \end{lemma}
% The proof is in Appendix \ref{app:thm:tildek}. 
  \begin{proof}
   When using $\bar{K}$ defined in \eqref{eq:tildek}, the matrix $Q$ in 
   \eqref{eq:dual_nobias} becomes $\bar{Q}$ as given below:  
   \begin{equation}
     \bar{Q}_{i,j} = \begin{cases}
       y_i y_j K(\bx_i, \bx_j),  & \text{ if } \pi(\bx_i) = \pi(\bx_j), \\
       0, & \text{ if } \pi(\bx_i) \neq \pi(\bx_j). 
     \end{cases}
   \end{equation}
   Therefore, the quadratic term in \eqref{eq:dual_nobias} can be decomposed into
   \begin{equation*}
     \AL^T \bar{Q} \AL = \sum_{c=1}^k \AL_{(c)}^T Q_{(c,c)} \AL_{(c)}.  
   \end{equation*}
   The constraints and linear term in \eqref{eq:dual_nobias} are also 
   decomposable, so the subproblems are independent,  
   and concatenation of 
   their optimal solutions, $\bar{\AL}$, is the optimal solution for \eqref{eq:dual_nobias} when $K$ is  
   replaced by $\bar{K}$.  
 \end{proof}

Based on the above lemma, we are able to bound  
 $\|\AL^*-\bar{\AL}\|_2$ by the sum of between-cluster kernel values:
 \begin{theorem}
   \label{thm:approx}
   Given data points $\bx_1, \dots, \bx_n$ and 
   a partition indicator $\{\pi(\bx_1),\dots,\pi(\bx_n)\}$, 
   \begin{equation}
     0\leq f(\bar{\AL}) -f(\AL^*) \leq  (1/2) C^2 D(\pi), 
     \label{eq:bound}
   \end{equation}
   where $f(\AL)$ is the objective function in \eqref{eq:dual_nobias}, 
   $\bar{\AL}$ is as in Lemma \ref{thm:tildek}, $\AL^*$ is the global optimal of \eqref{eq:dual_nobias} and
   $D(\pi)=\sum_{i,j: \pi(\bx_i)\neq \pi(\bx_j)} |K(\bx_i, \bx_j)|$. 
   Furthermore, 
   %if the kernel matrix has smallest eigenvalue $\sigma_n$, then
   $\|\AL^*-\bar{\AL}\|_2^2\leq C^2 D(\pi)/\sigma_n $
   where $\sigma_n$ is the smallest eigenvalue of the kernel matrix. 
 \end{theorem}
 \begin{proof}
   We use $\bar{f}(\AL)$ to denote the objective function of \eqref{eq:dual_nobias}
   with kernel $\bar{K}$. 
   By Lemma \ref{thm:tildek}, $\bar{\AL}$ is the minimizer of \eqref{eq:dual_nobias}
   with $K$ replaced by $\bar{K}$, 
   thus $\bar{f}(\bar{\AL}) \leq \bar{f}(\AL^*)$. 
   By the definition of $\bar{f}(\AL^*)$ we can easily show that 
   \begin{equation}
     \bar{f}(\AL^*) = f(\AL^*) - \frac{1}{2}\sum_{i,j: \pi(\bx_i)\neq \pi(\bx_j)} \alpha^*_i \alpha^*_j y_i y_j K(\bx_i, \bx_j) 
     \label{eq:eq_thm1_1}
   \end{equation}
   Similarly, we have 
   \begin{equation}
     \bar{f}(\bar{\AL}) = f(\bar{\AL}) - \frac{1}{2}\sum_{i,j: \pi(\bx_i)\neq \pi(\bx_j)}
     \bar{\alpha}_i \bar{\alpha}_j y_i y_j K(\bx_i,\bx_j), 
   \end{equation}
   Combining with $\bar{f}(\bar{\AL})\leq \bar{f}(\AL^*)$ we have
   \begin{align}
     f(\bar{\AL}) &\leq \bar{f}(\AL^*) + \frac{1}{2}\sum_{i,j: \pi(\bx_i)\neq \pi(\bx_j)} 
     \bar{\alpha}_i \bar{\alpha}_j y_i y_j K(\bx_i,\bx_j), \nonumber\\ 
     &= f(\AL^*) + \frac{1}{2}\!\!\sum_{i,j:\pi(\bx_i)\neq \pi(\bx_j)}\!\!\!\!\!\! (\bar{\alpha}_i \bar{\alpha_j}
     - \alpha^*_i \alpha^*_j) y_i y_j K(\bx_i,\bx_j)  \label{eq:111} \\ 
     &\leq f(\AL^*)\! +\! \frac{1}{2}C^2 D(\pi),  \text{ since } 0\leq \bar{\alpha}_i, \alpha^*_i\leq C \ \text{for all } i. \nonumber
   \end{align}
   Also, since $\AL^*$ is the optimal solution of \eqref{eq:dual_nobias}
   and $\bar{\AL}$ is a feasible solution, 
   $f(\AL^*)\leq f(\bar{\AL})$, thus proving the first part of the theorem. 

   Let $\sigma_n$ be the smallest singular value of the positive definite
   kernel matrix $K$. 
   Since $Q = \text{diag}(\by) K \text{diag}(\by)$, $Q$ and $K$ have identical singular values. 
%   has same smallest singular value
%   $\sigma_n$. 
   Suppose we write $\bar{\AL} = \AL^* + \Delta \AL$, 
   \begin{equation*}
     f(\bar{\AL}) = f(\AL^*) + (\AL^*)^T Q \Delta \AL + \frac{1}{2}(\Delta\AL)^T Q \Delta\AL -
     \be^T \Delta\AL. 
   \end{equation*}
   The optimality condition for \eqref{eq:dual_nobias} is 
   \begin{equation}
     \label{eq:thm1_opt}
     \nabla_i f(\AL^*)  \begin{cases}
      = 0 &\text{ if } 0 <\alpha^*_i <C, \\
      \geq 0 &\text{ if } \alpha^*_i = 0, \\
      \leq 0 & \text{ if } \alpha^*_i = C,  
     \end{cases}
   \end{equation}
   where $\nabla f(\AL^*) = Q\AL^* - \be$. 
   Since $\bar{\AL}$ is a feasible solution, it is easy to see that $(\Delta\AL)_i \geq 0$ if $\alpha^*_i=0$, 
   and $(\Delta\AL)_i \leq 0$ if $\alpha^*_i = C$. Thus, 
   \begin{equation*}
     (\Delta\AL)^T (Q \AL^*  - \be)
     = \sum_{i=1}^n (\Delta\AL)_i ((Q\AL^*)_i -1) \geq 0.  \\
   \end{equation*}
   So $f(\bar{\AL}) \ge f(\AL^*) + \frac{1}{2}\Delta\AL^T Q \Delta\AL \geq f(\AL^*) + 
   \frac{1}{2}\sigma_n \|\Delta\AL\|_2^2  
   $. Since we already know that $f(\bar{\AL})\leq f(\AL^*)+\frac{1}{2}C^2 D(\{\bx_i\}_{i=1}^n, \pi)
   $, this implies $\|\AL^*-\bar{\AL}\|_2^2 \leq C^2 D(\pi)/\sigma_n $.  
 \end{proof}

 In order to minimize $\|\AL^*-\bar{\AL}\|$, we want to find a partition with 
 small $D(\pi)$. Moreover, a balanced partition is preferred 
 to achieve faster training speed. 
 This can be done
 by the kernel kmeans 
 algorithm, which aims to minimize the off-diagonal values of the kernel matrix
 with a balancing normalization.
% We show in Figure \ref{fig:bound} that using kernel kmeans can achieve much smaller 
% $f(\bar{\AL})-f(\AL^*)$ compared to random partitioning. 

 %The proof given is in Appendix \ref{app:thm:approx}. 
 %In Appendix 
 %\ref{app:bound} we further show that $\frac{1}{2}C^2 D(\pi)$ is close to 
 %$f(\bar{\AL})-f(\AL^*)$ on real-life data, which indicates that bound \eqref{eq:bound} is tight.  
 We now show that
 the bound derived in Theorem \ref{thm:approx} is reasonably tight in practice. 
% strongly useful
% by conducting the following experiments. 
 On a subset (10000 instances) of the \covtype data, 
 we try different numbers of clusters $k=8,16,32,64,128$; 
 for each $k$, we use kernel kmeans to obtain the data partition 
 $\{\V_1,\dots,\V_k\}$, 
 and then compute $C^2D(\pi)/2$ (the right hand side of \eqref{eq:bound}) 
 and $f(\bar{\AL})-f(\AL^*)$ (the left hand side of \eqref{eq:bound}). 
 The results are presented in Figure \ref{fig:bound}.
 The left panel shows the bound (in red) and the difference in objectives $f(\bar{\AL})-f(\AL^*)$
 in absolute scale, while the right panel shows these values in a log scale. 
 Figure \ref{fig:bound} shows that the bound is quite close to the difference 
 in objectives in an absolute sense (the red and blue curves nearly overlap), 
 especially compared to the difference in objectives when the 
 data is partitioned randomly (this also shows effectiveness of the kernel kmeans 
 procedure). 
 Thus, our data partitioning scheme and subsequent solution of the subproblems 
 leads to good approximations to the global kernel SVM problem. 

% We found these two values are quite close for various $k$, which
% indicates the bound is pretty tight. 
% Moreover, we also plot $f(\bar{\AL})-f(\AL^*)$ when the partition is by 
% random, and Figure \ref{fig:bound} shows that partition by random is 
% extremely bad. 
% This supports that the clustering strategy for DC-SVM is useful. 
\begin{figure}[htb]
  \centering
\begin{tabular}{cc}
  \subfloat[\covtype 10000 samples. ]{\includegraphics[width=0.45\textwidth]{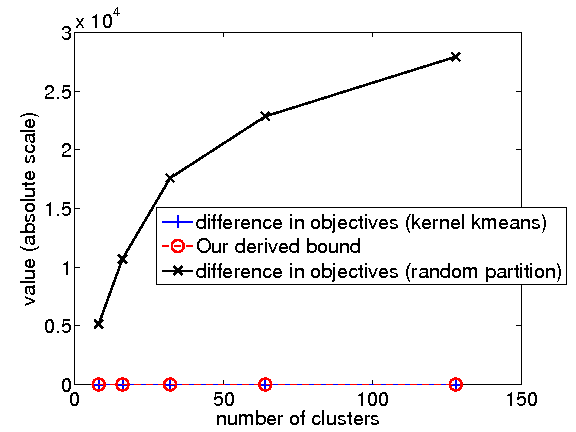}}
  &
  \subfloat[\covtype 10000 samples (log scale). ]{\includegraphics[width=0.45\textwidth]{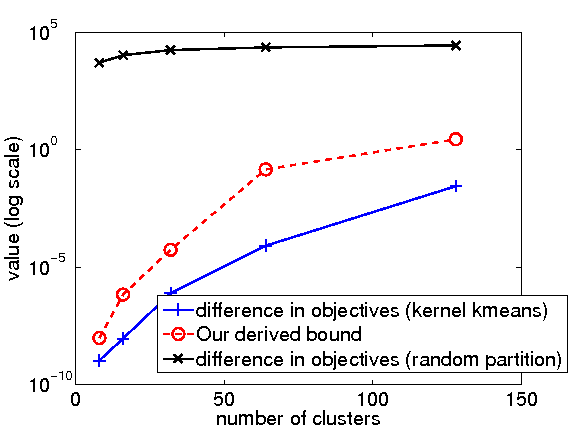}}
\end{tabular}
\caption{Demonstration of the bound in Theorem \ref{thm:approx} -- our data 
partitioning scheme leads to good approximations to the global solution $\AL^*$.
The left plot is on an absolute scale, while the right one is on a logarithmic scale. 
\label{fig:bound}  }
\end{figure}

\begin{figure*}[htb]
\begin{tabular}{cccc}
  \hspace{-35pt}
  \subfloat[rbf kernel, precision\label{fig:rbf_p} ]{\includegraphics[width=0.27\textwidth]{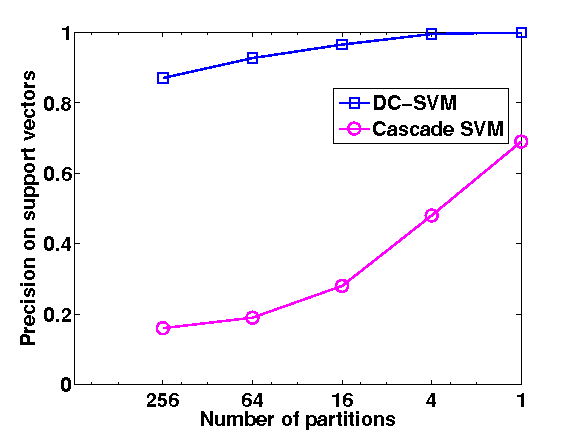}} & \hspace{-10pt}
  \subfloat[rbf kernel, recall\label{fig:rbf_r} ]{\includegraphics[width=0.27\textwidth]{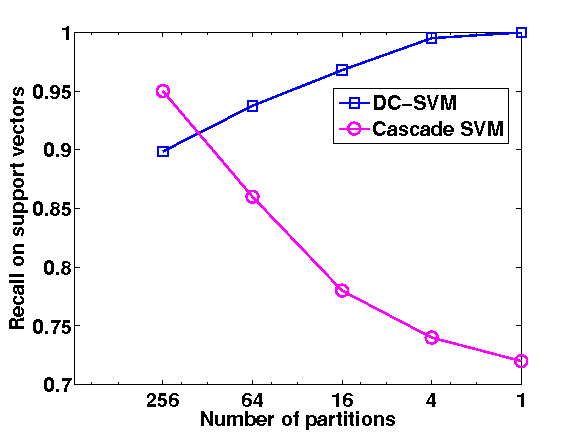}} &  \hspace{-10pt}
  \subfloat[rbf kernel, time vs. precision \label{fig:rbf_p_comp}]{\includegraphics[width=0.27\textwidth]{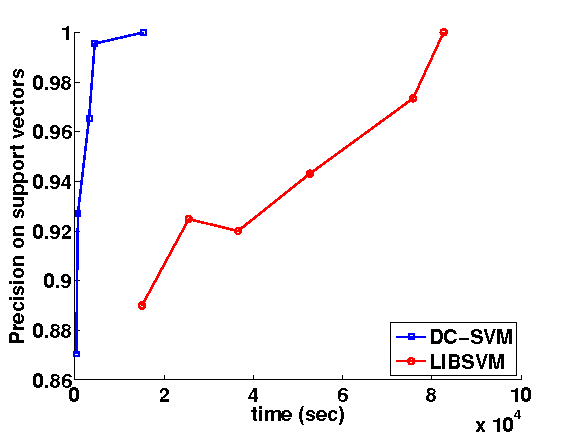}} & \hspace{-10pt}
  \subfloat[rbf kernel, time vs. recall\label{fig:rbf_r_comp} ]{\includegraphics[width=0.27\textwidth]{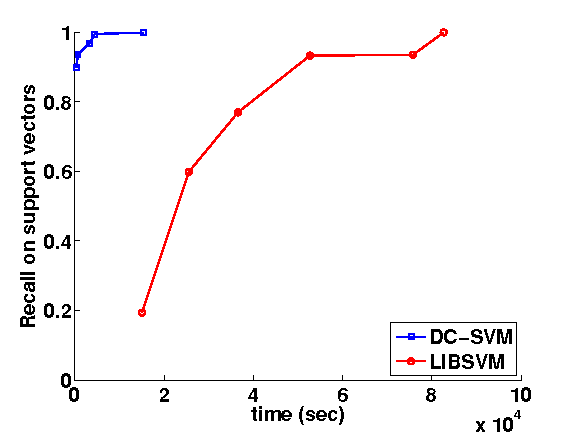}} 
  \\ 
  \hspace{-35pt}
  \subfloat[polynomial kernel, precision\label{fig:poly_p} ]{\includegraphics[width=0.27\textwidth]{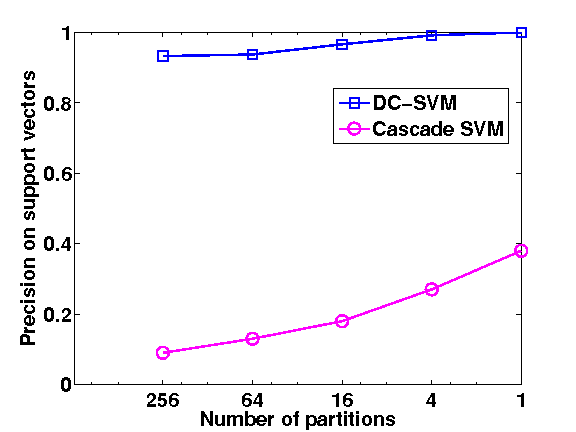}} & \hspace{-10pt}
  \subfloat[polynomial kernel, recall\label{fig:poly_r} ]{\includegraphics[width=0.27\textwidth]{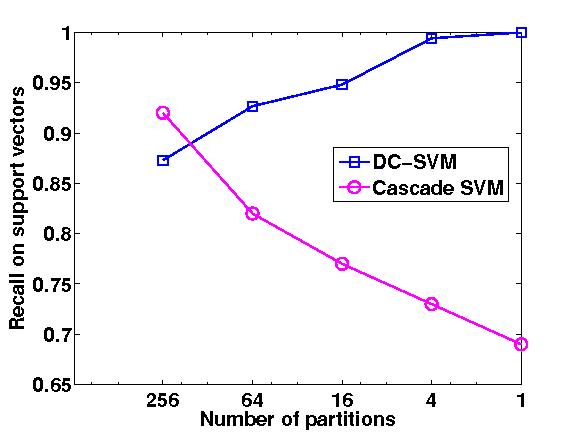}} & \hspace{-10pt}
  \subfloat[polynomial kernel,
  time vs. precision\label{fig:poly_p_comp} ]{\includegraphics[width=0.27\textwidth]{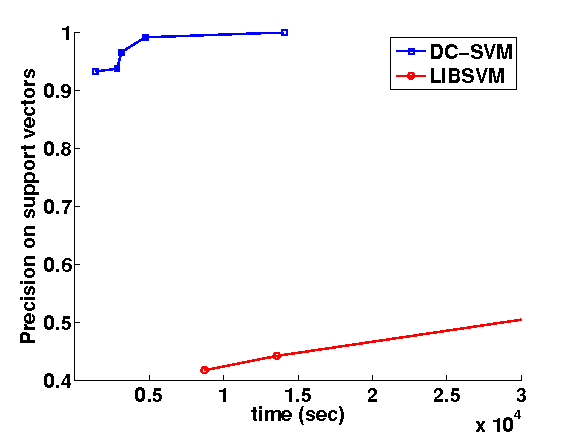}} & \hspace{-10pt}
\subfloat[polynomial kernel, time vs. recall \label{fig:poly_r_comp}]{\includegraphics[width=0.27\textwidth]{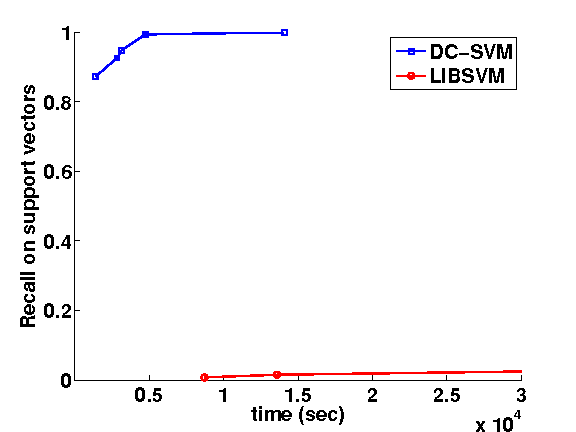}}
\end{tabular}
\caption{ 
Our multilevel DC-SVM algorithm computes support vectors for subproblems
during the ``conquer'' phase. The above plots show that DC-SVM identifies 
support vectors more accurately (Figure \ref{fig:rbf_p}, \ref{fig:rbf_r}, \ref{fig:poly_p}, 
\ref{fig:poly_r}) than cascade SVM, 
and more quickly than the shrinking strategy in \LIBSVM. 
%
%The precision and recall for DC-SVM to recover the true support vectors 
%on \covtype dataset. We test on both RBF kernel and polynomial kernels using 
%the parameters in Section \ref{sec:experiment}. 
%In Figure \ref{fig:rbf_p} and \ref{fig:poly_p}, 
%x-axis shows the number of clusters as we go up in the hierarchical tree in 
%DC-SVM,  
%and y-axis shows the precision of support vectors (compared to the support 
%vectors in the optimal
%solution). Figure \ref{fig:rbf_r} and \ref{fig:poly_r} presents the recall for each divide
%and conquer step. We can observe that the precision/recall are around 0.9 at the bottom level
%(256 clusters), and increases to $1$ as the algorithm goes up the hierarchical 
%tree. Also, DC-SVM can estimate support vectors more accurate than cascade SVM. Figure \ref{fig:rbf_p_comp}, \ref{fig:rbf_r_comp}, 
%\ref{fig:poly_p_comp}, and \ref{fig:poly_r_comp} compare the time for DC-SVM 
%and \LIBSVM to obtain accurate support vectors -- the results show that DC-SVM 
%significantly faster than \LIBSVM (shrinking strategy)
%in identifying support vectors. 
}
\label{fig:svbylevel}
\end{figure*}

 However, kernel kmeans has $O(n^2 d)$ time complexity, 
 %requires computation and storage of a size $n\times n$ 
 %dense kernel matrix, 
 which is too expensive for large-scale problems.
% Many approximate kernel kmeans algorithms have been proposed, for example, 
% kernel kmeans using Nystr\"{o}m approximation \citep{nystrom_kkmeans} 
% and random Fourier features \citep{kernelkmeans_rff}. 
Therefore we consider 
 a simple two-step kernel kmeans approach as in \citep{RC11a}. 
 The two-step kernel kmeans algorithm first runs kernel kmeans on $m$ randomly 
 sampled data points ($m\ll n$)
 to construct cluster centers in the kernel space. Based on these centers, 
 each data point computes its distance to cluster centers and decides which cluster it belongs to. The algorithm has time complexity $O(nmd)$ and space complexity $O(m^2)$.

% {\bf Quick identification of the support vectors. }
 A key facet of our proposed divide and conquer algorithm
 is that the set of support vectors from the subproblems 
 $\bar{S}:=\{i \mid \bar{\alpha}_i >0 \}$, where $\bar{\alpha}_i$ is the $i$-th 
 element of $\bar{\AL}$, is very close to that of the whole problem $S:=\{i\mid \alpha^*_i > 0\}$.  
%This is intuitively correct, because $\bar{\alpha}_i=0$ implies
%that $\bx_i$ is not close to the boundary of the within-cluster model, 
%and since
%$\AL^*$ is close to $\bar{\AL}$,  
%$\bx_i$ will not be close to the boundary even when training on the whole 
%dataset. 
Letting $\bar{f}(\AL)$ denote the objective function of \eqref{eq:dual_nobias}
with kernel $\bar{K}$ defined in \eqref{eq:tildek}, 
the following theorem shows that when 
$\bar{\alpha}_i=0$ ($\bx_i$ is not a support vector of the subproblem) and 
$\nabla_i \bar{f}(\bar{\AL})$
%, 
%which measures how far that instance from the boundary, 
is large enough, then $\bx_i$ will not be a support vector of the whole problem. 
%after computing the approximated solution $\bar{\AL}^{l}$ 
%in the $l$-th level, we can estimate $S$ by $\bar{S}^l = \{i\mid \bar{\alpha}^l_i>0\}$, 
%which is the support vectors obtained on the $l$-th level. 
\begin{theorem}
  \label{thm:support_vector}
  For any $i\in \{1,\dots,n\}$, if $\bar{\alpha}_i =0$ and 
%  $K_{max}=\max_i K(\bx_i,\bx_i)$ , then
  \begin{equation*}
    \nabla_i \bar{f}(\bar{\AL})> CD(\pi) (1+\sqrt{n}K_{max}/\sqrt{\sigma_n D(\pi) }),
%    \label{eq:sv_condition}
  \end{equation*}
  where $K_{max}=\max_i K(\bx_i,\bx_i)$,  
  then $\bx_i$ will not be a support vector of the whole problem, i.e., $\alpha^*_i=0$.
  \end{theorem}
\begin{proof}
  Let $\Delta Q = Q-\bar{Q}$ and $\Delta \AL = \AL^* - \bar{\AL}$. 
  From the optimality condition for \eqref{eq:dual_nobias} (see \eqref{eq:thm1_opt}), 
  we know that $\alpha^*_i =0$ if 
  $(Q\AL^*)_i > 1$. Since 
  $Q\AL^* = (\bar{Q}+\Delta Q)(\bar{\AL} + \Delta \AL)$, 
  we see that
  \begin{align*}
    (Q\AL^*)_i &= (\bar{Q}\bar{\AL})_i + (\Delta Q \bar{\AL})_i +  ( (\bar{Q}+\Delta Q) \Delta \AL)_i. \\
    &= (\bar{Q}\bar{\AL})_i + \sum_{j:\pi(\bx_i)\neq \pi(\bx_j)} y_iy_j K(\bx_i,\bx_j) \bar{\alpha}_j 
     + \sum_{j} y_i y_j K(\bx_i,\bx_j) (\Delta\AL)_j \\
    &\geq  (\bar{Q}\bar{\AL})_i  - C D(\pi) - K_{max} \|\Delta\AL\|_1 \\
    & \geq  (\bar{Q}\bar{\AL})_i  - C D(\pi) - \sqrt{n}K_{max} C \sqrt{D(\pi)} / \sqrt{\sigma_n} \\
    &= (\bar{Q}\bar{\AL})_i -CD(\pi) \left(1+\frac{\sqrt{n}K_{max}}{\sqrt{\sigma_n D(\pi) }}\right). 
  \end{align*}
%  Combining with Theorem \ref{thm:approx}, $(Q\AL^*)_i \geq (\bar{Q}\bar{\AL})_i - \frac{8\sqrt{n}K(\bx_i,\bx_i)}{\sigma_n} C^2 D(\pi) $.
  The condition stated in the theorem implies $(\bar{Q}\bar{\AL})_i  > 1+ 
  CD(\pi) (1+\frac{\sqrt{n}K_{max}}{\sqrt{\sigma_n D(\pi) }})$, 
  which implies 
  $(Q\AL^*)_i -1> 0$, so from the optimality condition \eqref{eq:thm1_opt}, $\alpha^*_i=0$. 
\end{proof}
%The proof is in Appendix \ref{app:thm:support_vector}. 
In practice also, we observe that DC-SVM can identify the set of support 
vectors of the whole problem very quickly. 
Figure \ref{fig:svbylevel} demonstrates that DC-SVM identifies support vectors
much faster than the shrinking strategy implemented in \LIBSVM \citep{CC01a}
(we discuss these results in more detail in Section \ref{sec:multilevel}).

 %The clustering algorithm
%\begin{itemize}
%  \item Want to divide the problem into subproblems such that 
%    (1) balanced size (2) The solution of subproblems are close to the 
%    solution of the whole problem.  
%  \item The solution of subproblems $ \AL^{(1)} \AL^{(2)} \dots \AL^{(k)}$ 
%    is the solution of the SVM problem with the block kernel. 
%  \item Theorem: $\|\alpha^* - \bar{\alpha}\|$ can be bounded by the 
%    sum of off-diagonal entries of the kernel matrix, where $\alpha^*$ is the 
%    optimal solution for the whole problem, and $\bar{\alpha}$ is the solution 
%    for the subproblems. So we should use {\bf kernel kmeans 
%    clustering.  }
%    \item Kernel kmeans is expensive, so we use two-step kernel kmeans \citep{RC11a}: 
%    conduct kernel kmeans on a subset of $m$ data to get the cluster centers, and then for all other data find the nearest centers. Time complexity $O(m^2d+knm)$ for clustering, and $O(nmd )$ to find the closet centers. $m$ cannot be too large!
%  \end{itemize}

  {\bf Conquer Step. }
  After computing $\bar{\AL}$ from the subproblems, we use $\bar{\AL}$ to initialize the solver for the whole problem. 
  In principle, we can use any SVM solver in our divide and conquer framework,  
  but we focus on using 
  coordinate descent method as in \LIBSVM%, which is a special case of SMO\citep{JP98a}, 
  %, which updates one variable at a time) 
  to solve the whole problem. 
 % The coordinate descent solver was implemented in \LIBSVM. 
  The main idea is to update one variable at a time, 
  and always choose the 
   $\alpha_i$ with the largest gradient value to update. 
%  In our divide and conquer algorithm, $\AL$ is not started from $0$, so we 
%  first compute the initial gradient $\nabla f(\bar{\AL})$ and then 
%  run \LIBSVM. 
  The benefit of applying coordinate descent is that we can avoid a 
  lot of unnecessary access to the kernel matrix entries if $\alpha_i$ never changes 
  from zero to nonzero.
  Since $\bar{\AL}$'s are close to $\AL^*$, the $\bar{\AL}$-values for most vectors 
  that are not support vectors will not become nonzero, 
  and so the algorithm converges quickly. 
  \section{Divide and Conquer SVM with multiple levels}
  \label{sec:multilevel}
  There is a trade-off in choosing the number of clusters $k$ for a single-level DC-SVM with only 
  one divide and conquer step. When $k$ 
  is small, the subproblems have similar sizes as the original problem, so 
 we will not gain much speedup.
  On the other hand, when we increase $k$, time complexity for solving subproblems can be reduced, 
  but the resulting $\bar{\AL}$ can be quite different from $\AL^*$ according to 
  Theorem \ref{thm:approx}, so the conquer step will be slow. 
% Although we can further reduce the time complexity by increasing the $k$, 
% the difference between $\bar{\AL}$ and ${\AL}^*$ will become larger according 
% to Theorem \ref{thm:approx}. 
Therefore, we propose to run DC-SVM with multiple levels to further reduce the 
time for solving the subproblems, and meanwhile still obtain $\bar{\AL}$ 
values that are close to $\AL^*$. 

In multilevel DC-SVM,   
at the $l$-th level, we 
 partition the whole dataset into $k^l$ clusters 
 $\{\V^{(l)}_1, \dots, \V^{(l)}_{k^l}\}$, and solve those $k^l$ subproblems independently 
 to get $\bar{\AL}^{(l)}$. In order to solve each subproblem efficiently, we use 
 the solutions from the lower level $\bar{\AL}^{(l+1)}$ to initialize the 
 solver at the $l$-th level, so each level requires very few 
 iterations. 
 This allows us to use small values of $k$, 
% By this way we can choose $k$ to be small; 
for example, we use $k=4$ for all the experiments. 
 In the following, we discuss more insights to further speed up our procedure. 

{\bf Adaptive Clustering. }
The two-step kernel kmeans approach has time complexity $O(nmd)$, so the number of samples $m$
cannot be too large. In our implementation we use $m=1000$. 
When the data set is very large, the performance of two-step kernel kmeans may 
not be good because 
we sample only a few data points. This will influence the 
performance of DC-SVM.

To improve the clustering for DC-SVM, we propose the following adaptive clustering 
approach. The main idea is to explore the sparsity of $\AL$ in the SVM problem, 
and sample from the set of support vectors to perform two-step kernel 
kmeans. 
The number of support vectors is  
generally much smaller than $n$ because of the bound
constraints. Suppose we 
are at the $l$-th level, and the current set of support vectors 
is defined by $\bar{S} = \{i\mid \bar{\alpha}_i > 0\}$. 
Suppose the set of support vectors for the final solution is given by $S^* = \{i \mid \alpha^*_i >0\}$. 
We can define the sum of off-diagonal elements on $\bar{S}\cup S^*$ as
 $D_{S^*\cup\bar{S}}(\pi)=\sum_{i,j\in S^*\cup\bar{S} \text{ and }\pi(\bx_i)\neq \pi(\bx_j)} |K(\bx_i, \bx_j)|$. 
The following theorem shows that we can refine the bound 
in Theorem \ref{thm:approx}:  
%know the set of support vectors $S$, 
%we can define the off-diagonal loss on $S$ by
%$D(\{\bx_i\}_{i\in S}, \pi) = \sum_{i,j\in S \text{ and } \pi(\bx_i)\neq \pi(\bx_j)}
%  |K(\bx_i, \bx_j)|. $
%  The following lemma shows that $D(\{\bx_i\}_{i=1}^n, \pi)=0$ implies $\bar{\AL} = \AL^*$. 
%  \begin{lemma}
%    \label{lm:optimal_sv}
%    Let $S=\{i\mid \alpha^*_i>0\}$, then $D_S(\pi)=0$ implies $\AL^*=\bar{\AL}$. 
%  \end{lemma}
%  Therefore, 
%  Even when $D_S(\pi)$ is not exact zero, we can bound $\|\AL^*-\bar{\AL}\|$ by 
%  the following theorem: 
  \begin{theorem}
    \label{thm:sv_kmeans}
       Given data points $\bx_1, \dots, \bx_n$ and 
   a partition $\{\V_1, \dots, \V_k\}$ with indicators $\pi$,  
   \begin{equation*}
     0\leq f(\bar{\AL})-f(\AL^*) \leq (1/2) C^2 D_{S^*\cup \bar{S}}(\pi ). 
   \end{equation*}
   Furthermore, $\|\AL^*-\bar{\AL}\|_2^2 \leq C^2 D_{S^*\cup\bar{S}}(\pi)/\sigma_n$. 
  \end{theorem}
  \begin{proof}
    Similar to the proof in Theorem \ref{thm:approx}, 
    we use $\bar{f}(\AL)$ to denote the objective function of \eqref{eq:dual_nobias}
    with kernel $\bar{K}$.  
    Combine \eqref{eq:111} with the fact that $\alpha^*_i=0 \ \ \forall i\notin S^*$
    and $\bar{\alpha}_i=0 \ \ \forall i\notin \bar{S}$, 
    we have
    \begin{align*}
      f(\bar{\AL}) & \leq  f(\AL^*)  - \frac{1}{2}  \sum_{i,j: \pi(\bx_i)\neq \pi(\bx_j) \text{ and } i,j\in S^*} (\bar{\alpha}_i\bar{\alpha}_j - \alpha^*_i \alpha^*_j)   y_i y_j K(\bx_i,\bx_j)\\
      & \leq f(\AL^*) + \frac{1}{2} C^2 D_{S^*\cup\bar{S}}(\pi). 
    \end{align*}
    The second part of the proof is identical to the proof of the second part of Theorem \ref{thm:approx}. 
  \end{proof}
  %The proof is provided in Appendix \ref{app:thm:sv_kmeans}. 
  The above observations suggest that if we know the set of support vectors $\bar{S}$
  and $S^*$, $\|\AL^*-\bar{\AL}\|$ only depends on whether we can obtain a good 
  partition of $\bar{S}\cup S^*$. 
  Therefore, we can sample $m$ points from $\bar{S}\cup S^*$ instead of the whole dataset to perform 
  the clustering. 
  The performance of two-step kernel kmeans depends on the sampling rate; 
  we reduce the sampling rate from $m/n$ to $m/|S^* \cup \bar{S}|$. 
  As a result, the performance 
  significantly improves when $|S^* \cup \bar{S} |\ll n$.

In practice we do not know $S^*$ or $\bar{S}$ before solving the problem. 
However, both Theorem \ref{thm:support_vector} and experiments shown in Figure \ref{fig:svbylevel} suggest that  
we have a good guess of support vectors even at the bottom level. 
Therefore, we can use
the lower level support vectors
as a good guess of the upper level support vectors.
%This motivates us to use
%the set of lower level support vectors to form a good guess 
%for $S^*$ and $\bar{S}$, and refine the two-step kernel kmeans. 
More specifically, 
after computing $\bar{\AL}^l$ from level $l$, we can use its support vector set $\bar{S}^l:=\{i\mid \bar{\alpha}_i^l > 0\}$ 
to run two-step kernel kmeans for finding the clusters at the $(l-1)$-th level.
Using this strategy, we obtain progressively better partitioning as we approach the original
problem at the top level.  

\begin{table*}
  \centering
  \caption{Comparing prediction methods using a lower level model. 
  Our proposed early prediction strategy is better in terms of prediction accuracy and testing 
  time per sample (time given in milliseconds).
  }
  \label{tab:prediction_method}
\resizebox{14cm}{!}{
  \begin{tabular}{|c|r|r|r|r|r|}
    \hline
   & \webspam $k=50$ & \webspam $k=100$ & \covtype $k=50$ & \covtype $k=100$ \\
    \hline
    Prediction by \eqref{eq:naive_prediction} & 92.6\% / 1.3ms   & 89.5\% / 
    1.3ms   & 94.6\% / 2.6ms & 92.7\% / 2.6ms \\
    BCM in \citep{Tresp} & 98.4\%  / 2.5ms  & 95.3\% / 3.3ms  & 91.5\% / 3.7ms  & 89.3\% / 5.6ms\\ 
%    Mixture SVM \citep{RC02a} & 98.7\% / 1.5ms & 97.2\% / 1.5ms & 93.6\% / 2.6ms & 93.0\% / 2.6ms  \\
    Early Prediction by \eqref{eq:predict_tildek} & {\bf 99.1\% / .17ms}  & 
    {\bf 99.0\% / .16ms }  &  {\bf 96.1\% / .4ms } & {\bf 96.0\% / .2ms 
    }\\
    \hline
  \end{tabular}
  }
\end{table*}

{\bf Early identification of support vectors. } 
We first run \LIBSVM to obtain the final set of support vectors, 
and then run DC-SVM with various numbers of clusters $4^5, 4^4, \dots, 4^0$ (corresponding
to level $5,4,\dots,1$ for multilevel DC-SVM). 
We show the precision and recall for the $\bar{\AL}$ at each level in identifying support vectors.
Figure \ref{fig:svbylevel} shows that DC-SVM can identify about 90\% support vectors
even using $256$ clusters. As discussed in Section \ref{sec:related}, Cascade SVM \citep{cascadeSVM} is another
way to identify support vectors. However, it is clear from Figure \ref{fig:svbylevel} that 
Cascade SVM cannot identify support vectors accurately as (1) it does not use kernel kmeans clustering, 
and (2) it cannot correct the false negative error made in lower levels.
Figure \ref{fig:rbf_p_comp}, \ref{fig:rbf_r_comp}, \ref{fig:poly_p_comp}, \ref{fig:poly_r_comp} further shows that DC-SVM identifies support vectors more quickly than the shrinking  strategy in \LIBSVM. 
%We further compare the speed for recovering support vectors for DC-SVM
%and \LIBSVM, and it shows the divide and conquer way to identify support vectors
%is much more efficient than shrinking. 

{\bf Early prediction based on the $l$-th level solution. }
Computing the exact kernel SVM solution can be quite time consuming, 
  so it is important to obtain a good model using limited time and memory.
  We now propose a way to efficiently predict the label of unknown instances using 
  the lower-level models $\bar{\AL}^l$. We will see in the experiments
  that prediction using $\bar{\AL}^l$ from a lower level $l$ already can achieve near-optimal testing performance. 

  When the $l$-th level solution $\bar{\AL}^l$ is computed, 
  a single naive way to predict a new instance $\bx$'s label $\tilde{y}$ is:  
%  A naive way to predict a new instance $\bx$ using the $l$-th level model 
%  of DC-SVM is to compute
  \begin{equation}
    \tilde{y} = \text{sign}\left(\sum_{i=1}^n y_i \bar{\alpha}^l_i K(\bx,\bx_i)\right). 
    \label{eq:naive_prediction}
  \end{equation}
  Another way to combine the models trained from $k$ clusters is to use 
  the probabilistic framework proposed in the Bayesian Committee
  Machine (BCM) \citep{Tresp}. 
  However, as we show below, both these methods do not give good prediction accuracy when the number of clusters
  is large. 

  Instead, we propose the following early prediction strategy. 
  %way to 
  %do prediction at an early stage. 
  From Lemma \ref{thm:tildek}, $\bar{\AL}$ is the optimal solution to 
  the SVM dual problem \eqref{eq:dual_nobias} on the whole dataset with the
  approximated kernel $\bar{K}$ defined in \eqref{eq:tildek}. 
  Therefore, we propose to use the same kernel function $\bar{K}$ in the testing phase, 
  which leads to the prediction
%  Using $\bar{K}$ as kernel, 
%  the decision value can be computed as
\begin{equation}
   \sum_{c=1}^k \sum_{i\in \V_c} y_i \alpha_i \bar{K}(\bx_i, \bx) = 
   \sum_{i\in \V_{\pi(\bx)}} y_i \alpha_i K(\bx_i, \bx), 
   \label{eq:predict_tildek}
 \end{equation}
 where $\pi(\bx)$ can be computed by finding the 
 nearest cluster center. 
 Therefore, the testing procedure for early prediction is: (1) find the nearest cluster that 
 $x$ belongs to, 
 and then (2) use the model trained by data within that cluster to 
 compute the decision value. 
 
 We compare this method with prediction by \eqref{eq:naive_prediction} and 
 BCM in Table \ref{tab:prediction_method}. 
 The results show that our proposed testing scheme is better in terms of test accuracy. 
 We also compare average testing time per instance in Table \ref{tab:prediction_method}, 
 and our proposed method is much more efficient as we only evaluate
 $K(\bx,\bx_i)$ for all $\bx_i$ in the same cluster as $\bx$, thus reducing the testing 
 time from $O(|S|d)$ to $O(|S|d/k)$, where $S$ is the set of support vectors.

{\bf Refine solution before solving the whole problem. }
  Before training the final model at the top level using the whole dataset, 
  we can refine the initialization by solving the SVM problem induced by all support
  vectors at the first level, i.e., level below the final level. 
  As proved in Theorem \ref{thm:support_vector}, the support vectors of lower level models 
  are likely to be the support vectors of the whole model, so
 this will give a more accurate solution, and only requires us  
 to solve a problem with $O(|\bar{S}^{(1)}|)$ samples, where $\bar{S}^{(1)}$ is the 
 support vectors at the first level. 
%  Notice that this approach is similar to the {\it shrinking} approach for SVM, where 
%  solvers actively guess variables which will not be support vectors, and then ignore 
%  those variables. We can also consider the whole DC-SVM process before solving 
%  the whole problem as a way to identify support vectors. As  
%  we will show in Section \ref{sec:experiment}, our divide-and-conquer approach is
%  much more efficient than shrinking (in \LIBSVM). 
%  
Our final algorithm is given in Algorithm \ref{alg:dc}.

  \begin{algorithm}[tbh]
  \caption{Divide and Conquer SVM}
  \label{alg:dc}
  \SetKwInOut{Input}{Input}\SetKwInOut{Output}{Output}
  \Input{Training data $\{(\bx_i,y_i)\}_{i=1}^n$, balancing parameter $C$, kernel function. }

  \Output{The SVM dual solution $\AL$. }

  \For{$l= l^{\text{max}},\dots,1$}{
  Set number of clusters in the current level $k_l = k^l$\;
  \uIf{$l=l^{\text{max}}$}{
     Sample $m$ points $\{\bx_{i_1},\dots,\bx_{i_m}\}$ from the whole training set\;
     }
     \Else{
     Sample $m$ points $\{\bx_{i_1},\dots,\bx_{i_m}\}$ from $\{\bx_i\mid \bar{\alpha}^{(l+1)}_i > 0\}$\;
     }
     
  Run kernel kmeans on $\{\bx_{i_1},\dots,\bx_{i_m}\}$ to get cluster centers $\bc_1,\dots,\bc_{k^l}$\; 
  Obtain partition $\V_1,\dots,\V_{k^l}$ for all data points \;
  \For{$c=1,\dots,k^l$}  {
  Obtain $\bar{\AL}^{(l)}_{\V_c}$ by solving SVM for the data in the $c$-th 
  cluster $\V_c$ with $\bar{\AL}^{(l+1)}_{\V_c}$ as the initial point (
  $\bar{\AL}^{l_{\text{max}}+1}_{\V_c}$ is set to 0)\; 
  }
  }
  Refine solution: Compute $\AL^{(0)}$ by solving  SVM on $\{\bx_i \mid \alpha^{(1)}_i\neq 0\}$ using $\AL^{(1)}$ as the initial point\;
  Solve SVM on the whole data using $\AL^{(0)}$ as the initial point\;
\end{algorithm}

\section{Experimental Results}
\label{sec:experiment}
%In this section, we empirically demonstrate the benefits of our proposed 
%method, DC-SVM. 

\begin{table}[inner sep=0pt]
  \centering
  \caption{Dataset statistics }
  \label{tab:datasets}
%\resizebox{7cm}{!}{
  \begin{tabular}{|c|r|r|r|} \hline
    \multirow{2}{*}{dataset}& Number of & Number of & \multirow{2}{*}{d} \\ 
    & training samples & testing samples &  \\
    \hline
    \ijcnn &49,990&91,701&22 \\
    \cifar &50,000 & 10,000 & 3072 \\
    \census & 159,619 & 39,904& 409 \\
    \covtype & 464,810 & 116,202 & 54 \\
    \webspam & 280,000 & 70,000 & 254 \\
\kddcup & 4,898,431 & 311,029 & 125 \\
\MNIST & 8,000,000 & 100,000 & 784 \\
    \hline
  \end{tabular}
%  }
\end{table}

%\begin{table}[inner sep=0pt]
%  \centering
%  \caption{Dataset statistics }
%  \label{tab:datasets}
%  \begin{tabular}{|c|r|r|r|r|r|} \hline
%    dataset&No. of training samples& No. of testing samples& d & C & $\gamma$\\ \hline
%    \ijcnn &49,990&91,701&22 & 32 & 2 \\
%    \cifar &50,000 & 10,000 & 3072 & 8 & $2^{-22}$ \\
%    \census & 159,619 & 39,904& 409 & 512 & $2^{-9}$\\
%    \covtype & 464,810 & 116,202 & 54 & 32 & 32 \\
%    \webspam & 280,000 & 70,000 & 254 & 8 & 32\\
%\kddcup & 4,898,431 & 311,029 & 125 & 256 & 0.5\\
%\MNIST & 8,000,000 & 100,000 & 784 & 1 & $2^{-21}$\\
%    \hline
%  \end{tabular}
%\end{table}

%\subsection{Real Datasets}
%In this section, 
We now compare our proposed algorithm with other 
SVM solvers.  
All the experiments are conducted 
on an Intel Xeon X5355 2.66GHz CPU with 8G RAM. 

{\bf Datasets: }
We use 7 benchmark datasets as shown in Table \ref{tab:datasets}\footnote{\cifar can be downloaded from \url{http://www.cs.toronto.edu/~kriz/cifar.html}; other datasets can be downloaded from \url{http://www.csie.ntu.edu.tw/~cjlin/libsvmtools/datasets} or the UCI data repository.}. 
We use the raw data without scaling for two image datasets \cifar and \MNIST, while features in all the other datasets 
are linearly scaled to $[0,1]$. \MNIST is a digital recognition dataset with 10 numbers, 
so we follow the procedure in \citep{ZK12a} to transform 
it into a binary classification problem by classifying round digits and non-round digits.  
Similarly, we transform \cifar into a binary classification problem
by classifying animals and non-animals. 
We use 
a random 80\%-20\% split for \covtype, \webspam, \kddcup, a random 8M/0.1M split for 
\MNIST (used in the original paper \citep{GL07b}), and the original training/testing split for \ijcnn and \cifar.
{\bf Competing Methods: }
We include the following exact SVM solvers (LIBSVM, CascadeSVM), approximate SVM solvers (SpSVM, LLSVM, FastFood, LTPU), and online SVM (LaSVM) in our comparison: 
%SVM with clustering (LTPU, LLSVM): 
%We include the following algorithms into comparison: 
\begin{compactenum}
%  \item linear SVM: the implementation in \liblinear\citep{REF08a}. 
\item LIBSVM: the implementation in the \LIBSVM library \citep{CC01a} with a small modification 
  to handle SVM without the bias term -- we observe that \LIBSVM has similar test accuracy with/without bias.
\item Cascade SVM: we implement cascade SVM \citep{cascadeSVM} using \LIBSVM as the base solver. 
\item SpSVM:
%  We did some modifications so that it can work on large datasets. }: 
Greedy basis selection for nonlinear SVM \citep{SSK06a}. 
%We adjusted the number 
%of basis vectors $b$ to control the time spent by SpSVM during the 
%experiments. 
\item LLSVM: improved Nystr\"{o}m method for nonlinear SVM by \citep{ZW11a}. 
  %We adjusted the rank of the approximated kernel
  %to control the time spent by LLSVM.
\item FastFood: use random Fourier features to approximate the kernel function \citep{fastfood}. 
  We solve the resulting linear SVM problem by the dual coordinate descent solver in \liblinear. 
  \item LTPU: Locally-Tuned Processing Units proposed in \citep{MJ89a}. 
    We set $\gamma$ equal to the 
    best parameter for Gaussian kernel SVM. The linear weights are obtained by 
    \liblinear.  
  \item LaSVM: An online algorithm proposed in \citep{AB05a}.
  \item DC-SVM: our proposed method for solving the exact SVM problem. 
  We use the modified \LIBSVM to solve subproblems. 
  \item DC-SVM (early): our proposed method with the early stopping approach
  described in Section \ref{sec:multilevel} to get the model before solving the entire kernel SVM optimization problem.
  \end{compactenum}
  \citep{ZK12a} reported that the low-rank approximation based method (LLSVM) outperforms 
  Core Vector Machines \citep{IWT05a} and the bundle method \citep{bmrm}, so we omit those comparisons here.
  Notice that we apply LIBSVM/LIBLINEAR as the default solver for DC-SVM, FastFood,  Cascade SVM, LLSVM and LTPU, 
  so the shrinking heuristic is automatically used in the experiments.  
\begin{table*}
  \centering
  \caption{
  Comparison on real datasets using the RBF kernel. }

\resizebox{15cm}{!}{
  \begin{tabular}{|c|rr|rr|rr|rr|}
    \hline
    &  \multicolumn{2}{|c}{\ijcnn} & \multicolumn{2}{|c}{\cifar} & \multicolumn{2}{|c}{\census} &  \multicolumn{2}{|c|}{\covtype
     } \\
    \cline{2-9} 
    & \multicolumn{2}{|c}{$C=32, \gamma=2$} & \multicolumn{2}{|c}{$C=8, \gamma=2^{-22}$} & \multicolumn{2}{|c}{$C=512, \gamma=2^{-9}$} & \multicolumn{2}{|c|}{$c=32, \gamma=32$} \\
    \hline
    & time(s) & acc(\%) & time(s) & acc(\%) & time(s) & acc(\%) & time(s) & acc(\%) \\
%    \hline
%    Linear SVM & 2 & 91.80 & 3 & 76.35 & 31 & 93.15 & 7204 & 75.04 \\
    \hline
    DC-SVM (early) & {\bf 12} & 98.35 & {\bf 1977} & 87.02 & {\bf 261} & {\bf 94.9} & {\bf 672} & 96.12 \\ 
    \hline
    DC-SVM & 41 & {\bf 98.69} & 16314 & {\bf 89.50} & 1051 & 94.2 & 11414 & {\bf 96.15}  \\
    \hline
    LIBSVM & 115 & {\bf 98.69} & 42688 & {\bf 89.50} & 2920 & 94.2 &  83631 & {\bf 96.15} \\    \hline
    LaSVM & 251 & 98.57 & 57204  & 88.19 & 3514 & 93.2 & 102603 & 94.39  \\ \hline
    CascadeSVM & 17.1 & 98.08 & 6148  & 86.8 & 849  & 93.0 & 5600 & 89.51 \\
    \hline
    LLSVM & 38 & 98.23 & 9745 & 86.5 &1212 & 92.8 & 4451 & 84.21 \\
    \hline
    FastFood & 87 & 95.95 & 3357 & 80.3 & 851 & 91.6 & 8550 & 80.1  \\
    \hline
    SpSVM & 20 & 94.92 & 21335 & 85.6& 3121 & 90.4 & 15113 & 83.37 \\
    \hline
    LTPU & 248 & 96.64 & 17418 & 85.3 & 1695 & 92.0 & 11532 & 83.25 \\
    \hline
  \end{tabular}
  }
  \label{tab:main_exp}
\end{table*}
\begin{table*}
  \centering
  \caption{
  Comparison on real datasets using the RBF kernel. }

\resizebox{12cm}{!}{
  \begin{tabular}{|c|rr|rr|rr|}
    \hline
      &  \multicolumn{2}{|c}{\webspam} & \multicolumn{2}{|c}{\kddcup}  & \multicolumn{2}{|c|}{\MNIST} \\
    \cline{2-7} 
    &  \multicolumn{2}{|c}{$C=8, \gamma=32$} & \multicolumn{2}{|c}{$C=256, \gamma=0.5$} & \multicolumn{2}{|c|}{$C=1, \gamma=2^{-21}$} \\
    \hline
  &  time(s) & acc(\%) & time(s) & acc(\%) & time(s) & acc(\%) \\
%    \hline
%    Linear SVM & 2 & 91.80 & 3 & 76.35 & 31 & 93.15 & 7204 & 75.04 \\
    \hline
    DC-SVM (early)  & {\bf 670} & 99.13 & {\bf 470} & {\bf 92.61} & {\bf 10287} & 99.85 \\ 
    \hline
    DC-SVM & 10485 & {\bf 99.28} & 2739 & 92.59 & 71823 & {\bf 99.93} \\
    \hline
    LIBSVM & 29472 & {\bf 99.28} & 6580 & 92.51 & 298900 & 99.91 \\    \hline
    LaSVM & 20342 & 99.25 & 6700 & 92.13 & 171400 & 98.95 \\ \hline
    CascadeSVM & 3515&98.1 & 1155 & 91.2 & 64151   & 98.3 \\
    \hline
    LLSVM  & 2853 & 97.74 & 3015 & 91.5 & 65121 & 97.64 \\
    \hline
    FastFood  & 5563 & 96.47 & 2191 & 91.6 & 14917& 96.5 \\
    \hline
    SpSVM & 6235 & 95.3 & 5124 & 90.5 &121563 & 96.3\\
    \hline
    LTPU  & 4005 & 96.12 & 5100 & 92.1 & 105210 & 97.82 \\
    \hline
  \end{tabular}
  }
  \label{tab:main_exp2}
\end{table*}

  {\bf Parameter Setting: }
 We first consider the RBF kernel 
$K(\bx_i,\bx_j) = \exp(-\gamma \|\bx_i-\bx_j\|_2^2)$. 
We chose the balancing parameter $C$ and kernel parameter $\gamma$ by 5-fold cross 
validation on a grid of points: $C=[2^{-10}, 2^{-9}, \dots, 2^{10}]$ and  
$\gamma = [2^{-10}, \dots, 2^{10}]$ for \ijcnn, \census, \covtype, \webspam, and \kddcup. 
The average distance between samples for un-scaled image datasets \MNIST and \cifar is much larger than other datasets, 
so we test them on smaller $\gamma$'s: $\gamma=[2^{-30}, 2^{-29}, \dots, 2^{-10}]$. 
Regarding the parameters for DC-SVM, we use 5 levels ($l^{\text{max}}=4$) and 
$k=4$, so the five levels have $1, 4, 16, 64$ and $256$ clusters respectively. 
For DC-SVM (early), we stop at the level with 64 clusters.
The following are parameter settings for other methods in Table \ref{tab:main_exp}:
the rank is set to be 3000 in LLSVM; 
number of Fourier features is 3000 in Fastfood\footnote{In Fastfood we control the number of blocks so that number of Fourier features is close to 3000 for each dataset. }; 
number of clusters is 3000 in LTPU; 
number of basis vectors is 200 in SpSVM; the tolerance in the stopping condition for LIBSVM and DC-SVM is set to 
$10^{-3}$ (the default setting of \LIBSVM); for LaSVM we set the number of passes to be 1; 
for CascadeSVM we output the results after the first round. 
%for LLSVM, we set the rank to be 3000. For LTPU, we set number of clusters to be 3000. 
%Due to the space limitations, we have placed more detailed experimental results in Appendix 
%\ref{app:figures}. 
\begin{figure*}
  \begin{center}
\begin{tabular}{ccc}
  \subfloat[\webspam objective function \label{fig:rbf_webspam_obj}]{\includegraphics[width=0.33\textwidth]{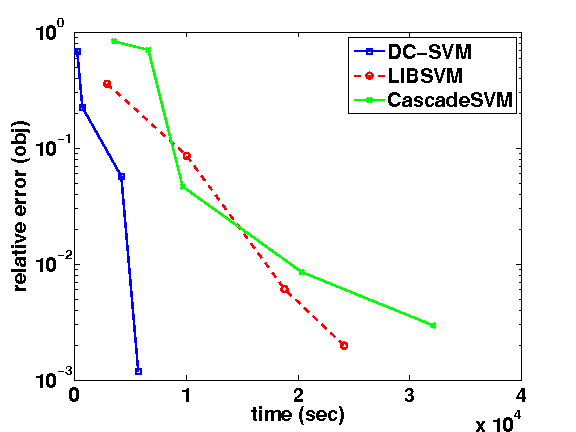}}
&
\subfloat[\covtype objective function \label{fig:rbf_covtype_obj}]{\includegraphics[width=0.33\textwidth]{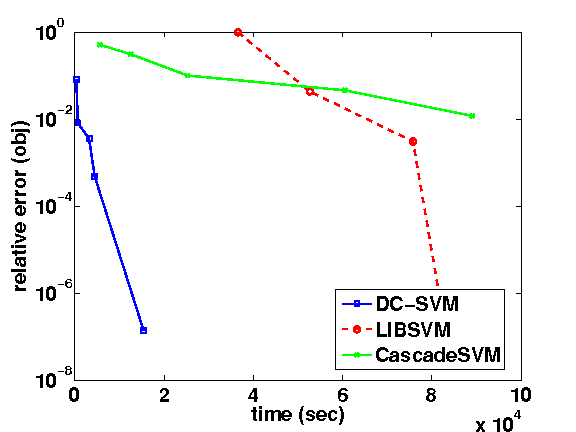}} & 
\subfloat[\MNIST objective function \label{fig:rbf_mnist8m_obj}]{\includegraphics[width=0.33\textwidth]{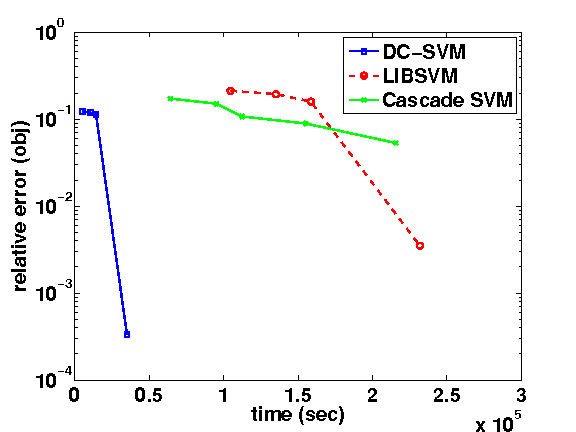}} \\
 \subfloat[\webspam testing accuracy \label{fig:rbf_webspam_acc}]{\includegraphics[width=0.33\textwidth]{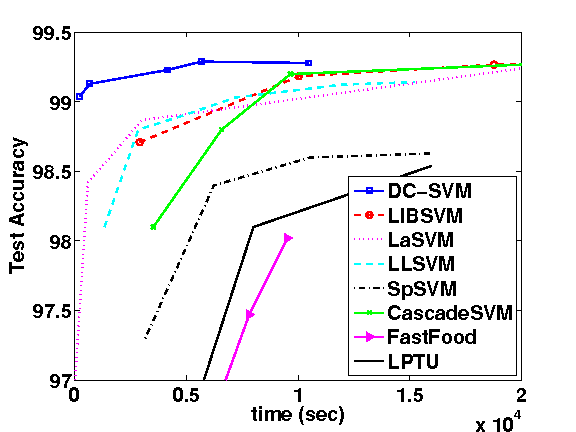}} &
\subfloat[\covtype testing accuracy \label{fig:rbf_covtype_acc}]{\includegraphics[width=0.33\textwidth]{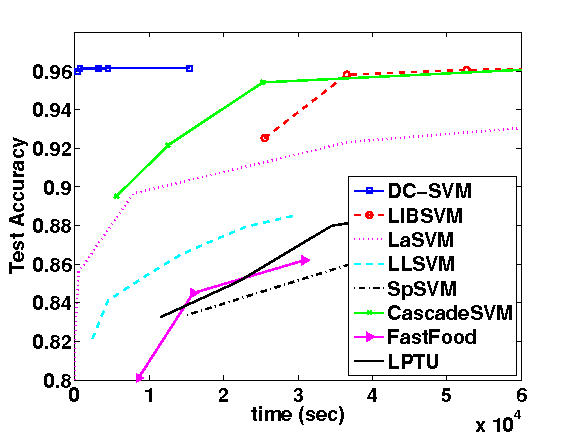}} &
\subfloat[\MNIST testing accuracy \label{fig:rbf_mnist8m_acc}]{\includegraphics[width=0.33\textwidth]{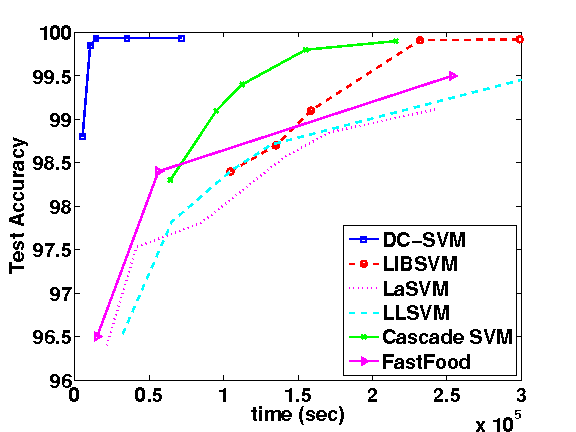}} 
\\
\subfloat[\kddcup objective function \label{fig:rbf_kddcup_obj}]{\includegraphics[width=0.33\textwidth]{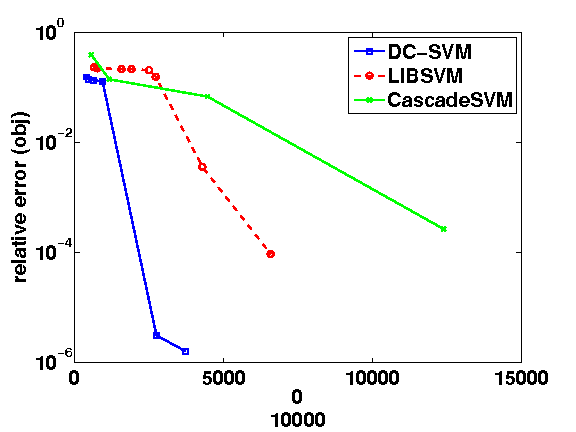}}
&
\subfloat[\cifar objective function \label{fig:rbf_cifar_obj}]{\includegraphics[width=0.33\textwidth]{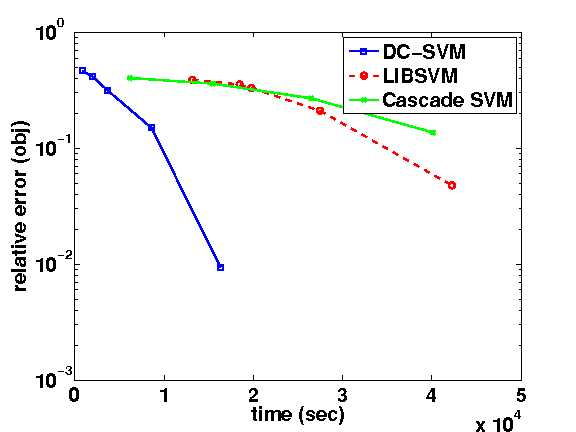}} 
& 
\\
\subfloat[\kddcup testing accuracy \label{fig:rbf_kddcup_acc}]{\includegraphics[width=0.33\textwidth]{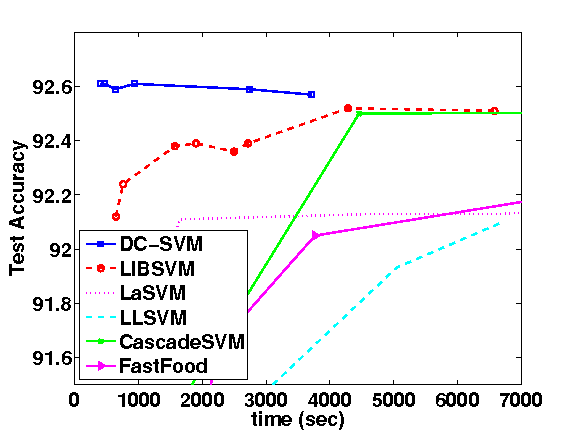}}
&
\subfloat[\cifar testing accuracy \label{fig:rbf_cifar_acc}]{\includegraphics[width=0.33\textwidth]{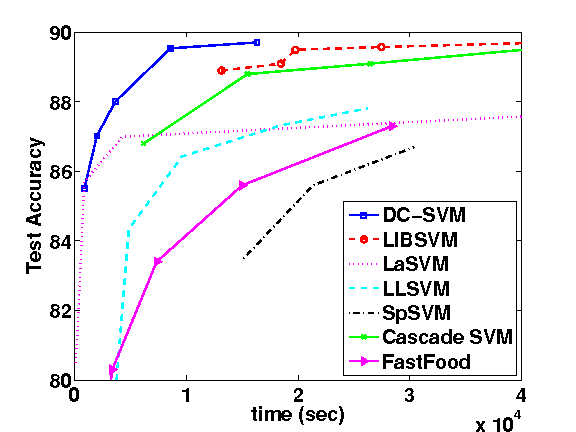}} & 
\\
\end{tabular}
\end{center}
\caption{ Comparison of algorithms using the RBF kernel. 
Each point for DC-SVM indicates the result when stopping at different levels; 
each point for \LIBSVM and CascadeSVM indicates different stopping conditions; each point for LaSVM indicates various number of passes
through data points; each point for LTPU and LLSVM, and FastFood indicates different sample sizes; and each point for SpSVM indicates
different number of basis vectors. Methods with
testing performance below the bottom of the y-axis are not shown in the figures. 
}
\label{fig:exp_time_rbf}
\end{figure*}
% \begin{figure*}
%  \begin{center}
%\begin{tabular}{cc}
%\subfloat[\kddcup objective function \label{fig:rbf_kddcup_obj}]{\includegraphics[width=0.38\textwidth]{svm_figs/kdd_obj_new.png}}
%&
%\subfloat[\cifar objective function \label{fig:rbf_cifar_obj}]{\includegraphics[width=0.38\textwidth]{dcsvm_figs/cifar_obj.png}} \\
%\subfloat[\kddcup testing accuracy \label{fig:rbf_kddcup_acc}]{\includegraphics[width=0.38\textwidth]{dcsvm_figs/kdd_acc_new.png}}
%&
%\subfloat[\cifar testing accuracy \label{fig:rbf_cifar_acc}]{\includegraphics[width=0.38\textwidth]{dcsvm_figs/cifar_acc.png}} \\
%
%\end{tabular}
%\end{center}
%\caption{Additional comparison of algorithms using RBF kernel on the \kddcup and \cifar datasets. 
%}
%\label{fig:exp_time_rbf_app}
%\end{figure*}

{\bf Experimental Results with RBF kernel: }
Table \ref{tab:main_exp} presents time taken and test accuracies. 
%the time and test accuracy on Gaussian kernel. 
%For DC-SVM (early), we stop at the 
%level with 64 clusters. 
Experimental results show that the early prediction
approach in DC-SVM achieves near-optimal test performance. By going to the
top level (handling the whole problem), DC-SVM achieves better test performance
but needs more time. 
Table \ref{tab:main_exp} only gives the comparison on {\it one} setting; 
it is natural to ask, for example, about the performance of LIBSVM with a looser stopping condition, 
or Fastfood with varied number of Fourier features. Therefore, for each algorithm we change
the parameter settings and present more detailed results in Figure \ref{fig:exp_time_rbf}. 
%we further compare all the methods with different stopping condition and sample sizes
%in Figure \ref{fig:exp_time_rbf}. 
%\LIBSVM and LaSVM are slower comparing to DC-SVM.  
%Figure \ref{fig:exp_time_rbf} presents thorough experimental results on RBF kernel. 

%In Figure \ref{fig:exp_time_rbf}, each point of DC-SVM indicates the result when stopping at different levels; 
%each point of \LIBSVM and CascadeSVM indicates different stopping conditions; each point of LaSVM indicates various number of passes
%through data points; each point of LTPU and LLSVM, and FastFood indicates different sample sizes; and each point of SpSVM indicates
%different number of basis vectors. Methods with
%testing performance below the bottom of y-axis are not shown in the figures. 
 
Figure \ref{fig:exp_time_rbf} shows convergence results with time -- in \ref{fig:rbf_webspam_obj}, \ref{fig:rbf_covtype_obj}, \ref{fig:rbf_mnist8m_obj}  
the relative error on the y-axis is defined as 
$(f(\AL)-f(\AL^*))/f(\AL^*)$, where $\AL^*$ is computed by running LIBSVM with $10^{-8}$ accuracy. 
Online and approximate solvers are not included in this comparison as they do not solve the
exact kernel SVM problem. 
We observe that DC-SVM achieves faster convergence in objective function 
compared with the 
state-of-the-art exact SVM solvers. 
%, \LIBSVM and cascade SVM. 
%The results are presented in Figure \ref{fig:rbf_webspam_obj}, \ref{fig:rbf_covtype_obj}, 
%and \ref{fig:rbf_mnist8m_obj}. 
%(also see Fig \ref{fig:rbf_webspam_obj} in Appendix). 
%compare the training speed of DC-SVM and \LIBSVM 
%for reducing the objective function value. 
%, where the optimal solution $\AL^*$ is obtained 
%by running \LIBSVM with stopping tolerance $10^{-4}$. 
%We can observe that 
%DC-SVM outperforms \LIBSVM in terms of time taken on all datasets, especially on \covtype and \MNIST.  
Moreover, DC-SVM is 
%not only efficient in decreasing the objective function value compared with 
%the exact SVM solver, but is 
also able to  
achieve superior test accuracy in lesser  training time as 
compared with approximate solvers. Figure \ref{fig:rbf_webspam_acc}, 
\ref{fig:rbf_covtype_acc}, \ref{fig:rbf_mnist8m_acc} 
%(also see Fig. 
%\ref{fig:rbf_webspam_acc} in Appendix)
compare the efficiency in achieving different testing accuracies.
We can see that DC-SVM consistently achieves more than 50 fold speedup
while achieving higher testing accuracy.  

{\bf Experimental Results with varying values of $C, \gamma$: }
As shown in Theorem \ref{thm:approx} the quality of approximation depends on $D(\pi)$, which is strongly related
to the kernel parameters. In the RBF kernel, when $\gamma$ is large, a large portion of kernel entries will be close
to 0, and $D(\pi)$ will be small so that $\bar{\alpha}$ is a good initial point for the top level. 
On the other hand, when $\gamma$ is small, $\bar{\alpha}$ may not be close to the optimal solution. 
To test the performance of DC-SVM under different parameters, 
we conduct the comparison on a wide range of parameters 
($C=[2^{-10}, 2^{-6}, 2^{1}, 2^6, 2^{10}], \gamma=[2^{-10},2^{-6}, 2^1, 2^6,  2^{10}]$). The results on the \ijcnn, \covtype, \webspam and \census
datasets are shown in Tables \ref{tab:c_gamma_ijcnn}, \ref{tab:c_gamma_webspam},  
\ref{tab:c_gamma_covtype}, \ref{tab:c_gamma_census} (in the appendix).
We observe that even when $\gamma$ is small, DC-SVM is still 1-2 times faster than LIBSVM: among all the 100 settings, 
DC-SVM is faster on 96/100 settings. The reason is that even when $\bar{\alpha}$ is not close to $\alpha$, 
using $\bar{\alpha}$ as the initial point is still better than initialization with a random or zero vector.   
On the other hand, DC-SVM (early) is extremely fast, and
achieves almost the same or even better accuracy when $\gamma$ is small (as it uses an approximated kernel). 
In Figure \ref{fig:parameter_plot_ijcnn}, \ref{fig:parameter_plot_webspam}, \ref{fig:parameter_plot_covtype}, \ref{fig:parameter_plot_census} (in appendix) we plot the performance of DC-SVM and LIBSVM under various $C$ and $\gamma$ values,
the results indicate that DC-SVM (early) is more robust to parameters.  
The accumulated runtimes are shown in Table \ref{tab:grid_sum}.

{\bf Experimental Results with the polynomial kernel: }
To show that DC-SVM is efficient for different types of kernels, we further 
conduct experiments on \covtype and \webspam datasets for the degree-3 polynomial 
kernel $K(\bx_i,\bx_j)=(\eta + \gamma \bx_i^T \bx_j)^3$. 
For the polynomial kernel, 
the parameters chosen by cross validation are $C=2, \gamma=1$ for \covtype, and $C=8, \gamma=16$ for \webspam. 
We set $\eta=0$, which is the default setting in \LIBSVM. 
Figures \ref{fig:poly_web_obj} and \ref{fig:poly_cov_obj} compare the 
training speed of DC-SVM and \LIBSVM 
for reducing the objective function value while Figures \ref{fig:poly_web_acc} 
and \ref{fig:poly_cov_acc} show the testing accuracy compared with \LIBSVM and LaSVM. 
Since LLSVM, FastFood and LPTU are developed for shift-invariant kernels, we do not include them in our 
comparison. We can see that when using the polynomial kernel, our algorithm is more than 100 times faster than \LIBSVM and LaSVM. One main reason for such large improvement is that it is hard for \LIBSVM and LaSVM to identify the right
set of support vectors when using the polynomial kernel. As shown in Figure \ref{fig:svbylevel}, \LIBSVM cannot 
even identify 20\% of the support vectors in $10^5$ seconds, while DC-SVM has a 
very good guess of the support vectors even at the bottom level, 
where number of clusters is 256.
  \begin{figure*}
  \begin{center}
    \vspace{-20pt}
\begin{tabular}{cc}
  \subfloat[\webspam objective function \label{fig:poly_web_obj}]{\includegraphics[width=0.4\textwidth]{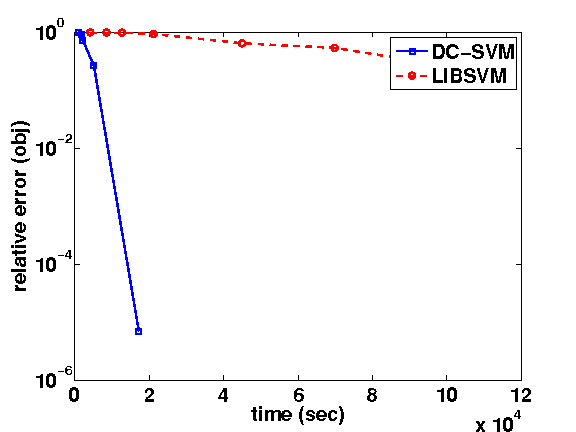}} 
  & 
\subfloat[\webspam testing accuracy \label{fig:poly_web_acc}
]{\includegraphics[width=0.4\textwidth]{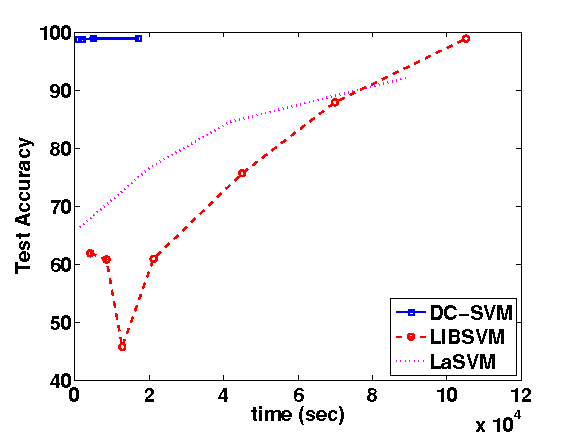}}\vspace{-10pt} \\
\subfloat[\covtype objective function\label{fig:poly_cov_obj}]{\includegraphics[width=0.4\textwidth]{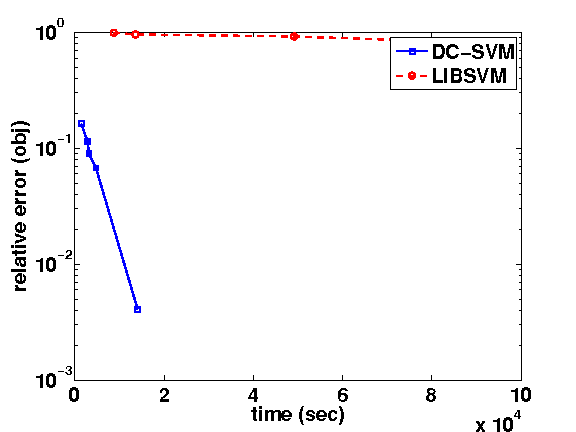}}
& 
\subfloat[\covtype testing accuracy\label{fig:poly_cov_acc}]{\includegraphics[width=0.4\textwidth]{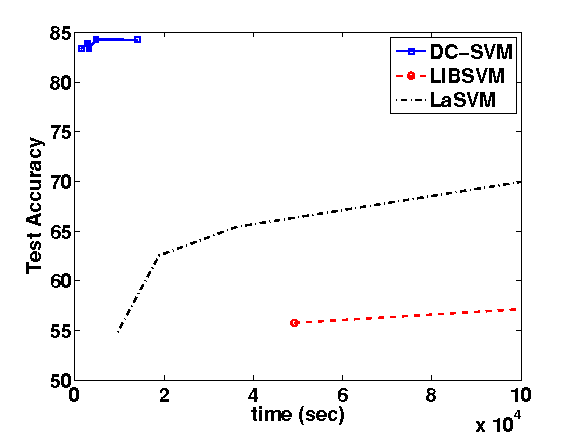}} 
\vspace{-10pt}
\end{tabular}
\end{center}
\caption{ Comparison of algorithms on real datasets using the polynomial kernel.  
}
\vspace{-5pt}
\label{fig:exp_time_poly}
\end{figure*}

{\bf Clustering time vs Training time. }
Our DC-SVM algorithm is composed of two important parts: clustering and SVM training. 
In Table \ref{tab:runtime} we list the time taken by each part; 
we can see that the clustering time is almost constant at each level, while the rest of the training
time keeps increasing.   
\begin{table}[inner sep=0pt]
  \centering
  \caption{Total time for DC-SVM, DC-SVM (early) and \LIBSVM on the grid of parameters $C, \gamma$ shown in Tables \ref{tab:c_gamma_ijcnn}, \ref{tab:c_gamma_webspam}, \ref{tab:c_gamma_covtype}, and \ref{tab:c_gamma_census} (in the Appendix). }
  \vspace{-5pt}
  \label{tab:grid_sum}
%\resizebox{13cm}{!}{
  \begin{tabular}{|c|r|r|r|}
    \hline
    \multirow{2}{*}{dataset} & DC-SVM  & \multirow{2}{*}{DC-SVM}  & \multirow{2}{*}{\LIBSVM} \\ 
    & (early) &  &  \\
    \hline
    \ijcnn & 16.4 mins & 2.3 hours & 6.4 hours \\ % & 20,210 \\
    \hline
    \webspam & 5.6 hours & 4.3 days & 14.3 days  \\
    \hline
    \covtype & 10.3 hours & 4.8 days & 36.7 days  \\
    \hline
    \census & 1.5 hours & 1.4 days & 5.3 days \\
    \hline
  \end{tabular}
  \vspace{-5pt}
%  }
\end{table}

\begin{table}[h]
  \centering
  \caption{Run time (in seconds) for DC-SVM on different levels (\covtype dataset). We can see the clustering time is only a small portion compared with the total training time. \label{tab:runtime}}
  \begin{tabular}{|r|r|r|r|r|r|}
    \hline
  Level & 4 & 3 & 2 & 1 & 0 \\
  \hline
  Clustering & 43.2s & 42.5s & 40.8s & 38.1s & 36.5s \\
  \hline
  Training &  159.4s & 439.7s & 1422.8s & 3135.5s & 7614.0s \\
  \hline
  \end{tabular}
\end{table}

%\begin{figure}[htb]
%\begin{tabular}{cc}
%\subfigure[webspam obj]{\includegraphics[width=0.48\textwidth]{dcsvm_figs/webspam_obj.png}}
%&
%\subfigure[webspam acc]{\includegraphics[width=0.48\textwidth]{dcsvm_figs/webspam_acc.png}}
%\\
%\subfigure[\kddcup obj]{\includegraphics[width=0.48\textwidth]{dcsvm_figs/kdd_obj.png}}
%&
%\subfigure[\kddcup acc]{\includegraphics[width=0.48\textwidth]{dcsvm_figs/kdd_acc.png}}
%\\
%\subfigure[covtype obj]{\includegraphics[width=0.48\textwidth]{dcsvm_figs/covtype_obj.png}}
%&
%\subfigure[covtype acc]{\includegraphics[width=0.48\textwidth]{dcsvm_figs/covtype_acc.png}}
%\\
%\subfigure[mnist8m obj]{\includegraphics[width=0.48\textwidth]{dcsvm_figs/mnist_obj.png}}
%&
%\subfigure[mnist8m acc]{\includegraphics[width=0.48\textwidth]{dcsvm_figs/mnist_acc.png}}
%\\
%\end{tabular}
%\caption{Comparison of algorithms on real datasets.  
%}
%\label{fig:exp_time}
%\end{figure}

%  \section{Others}
%  \begin{itemize}
%    \item Do we want to add BKA-LR here?
%    \item Whether we want to show the connection between kmeans and kernel 
%      off-diagonal elements in this paper or in the kernel approximation 
%      paper. 
%    \end{itemize}

      \section{Conclusions}
      \label{sec:conclusion}
      In this paper, we have proposed a novel divide and conquer algorithm for solving
      kernel SVMs (DC-SVM). Our algorithm divides the problem into smaller subproblems
      that can be solved independently and efficiently. We 
      show that the subproblem solutions are close to that of the original 
      problem, which motivates us to ``glue'' solutions from subproblems in order to 
      efficiently solve the original kernel SVM problem. 
      Using this, we also incorporate an early prediction strategy into our 
      algorithm. 
      We report extensive experiments to demonstrate that DC-SVM significantly outperforms 
      state-of-the-art exact 
      and approximate solvers for nonlinear kernel SVM on large-scale datasets. The code for DC-SVM
      is available at \url{http://www.cs.utexas.edu/~cjhsieh/dcsvm}. 

\bibliography{cluster_svm}
\bibliographystyle{plainnat}
\newpage
%  \subsection{Experimental results on a grid of $C, \gamma$}
\begin{table*}[inner sep=0pt]
  \vspace{-20pt}
  \centering
  \caption{Comparison of DC-SVM, DC-SVM (early), and LIBSVM on \ijcnn with various parameters $C, \gamma$. DC-SVM (early) is always 10 times faster than LIBSVM and achieves similar test accuracy. DC-SVM is faster than LIBSVM for almost every setting.  }
  \label{tab:c_gamma_ijcnn}
\resizebox{16cm}{!}{
  \begin{tabular}{|c|r|r|rr|rr|rr|rr|}
    \hline
    \multirow{2}{*}{dataset} & \multirow{2}{*}{$C$} & \multirow{2}{*}{$\gamma$} & \multicolumn{2}{|c}{DC-SVM (early)} & 
    \multicolumn{2}{|c}{DC-SVM}  & \multicolumn{2}{|c}{\LIBSVM} & 
    \multicolumn{2}{|c|}{LaSVM}\\ 
    \cline{4-11}
  & & & acc(\%) & time(s) & acc(\%) & time(s) & acc(\%) & time(s) & acc(\%) & 
   time(s) \\
    \hline
    \ijcnn & $2^{-10}$ & $2^{-10}$ & {\bf 90.5}  & {\bf 12.8}  & {\bf 90.5}  & 120.1   & {\bf 90.5}  & 130.0 & {\bf 90.5}  & 492 \\
    \ijcnn & $2^{-10}$ & $2^{-6}$ & {\bf 90.5} & {\bf 12.8} & {\bf 90.5} & 203.1  & {\bf 90.5} & 492.5 & {\bf 90.5} & 526 \\
    \ijcnn & $2^{-10}$ & $2^{1}$ & {\bf 90.5}  & {\bf 50.4} & {\bf 90.5} &524.2   & {\bf 90.5}& 1121.3 & {\bf 90.5} & 610 \\
    \ijcnn & $2^{-10}$ & $2^{6}$ & {\bf 93.7} & {\bf 44.0} & {\bf  93.7}& 400.2  & {\bf 93.7}  & 1706.5 & 92.4 & 1139  \\
    \ijcnn & $2^{-10}$ & $2^{10}$ & {\bf 97.1}  & {\bf 39.1} & {\bf 97.1}  & 451.3  & {\bf 97.1}  & 1214.7 & 95.7  & 1711  \\
    \ijcnn & $2^{-6}$ & $2^{-10}$ & {\bf 90.5}  & {\bf 7.2}  & {\bf 90.5}  & 84.7  &{\bf 90.5}  & 252.7 & {\bf 90.5} &531  \\
    \ijcnn & $2^{-6}$ & $2^{-6}$ & {\bf 90.5} & {\bf 7.6} & {\bf 90.5}  & 161.2 &{\bf 90.5}  & 401.0 & {\bf 90.5} & 519 \\
    \ijcnn & $2^{-6}$ & $2^{1}$ & 90.7  & {\bf 10.8} & {\bf 90.8} & 183.6 & {\bf 90.8} & 553.2 & 90.5 & 577 \\
    \ijcnn & $2^{-6}$ & $2^{6}$ & {\bf 93.9} & {\bf 49.2} & {\bf 93.9} & 416.1 & {\bf 93.9}  & 1645.3 & 91.3 & 1213 \\
    \ijcnn & $2^{-6}$ & $2^{10}$ & {\bf 97.1}  & {\bf 40.6} & {\bf 97.1}  & 477.3 & {\bf 97.1} & 1100.7 & 95.5 & 1744 \\
    \ijcnn & $2^{1}$ & $2^{-10}$ & {\bf 90.5}  & {\bf 14.0}  & {\bf 90.5} & 305.6   & {\bf 90.5}  & 424.9 & {\bf 90.5} &511 \\
    \ijcnn & $2^{1}$ & $2^{-6}$ & 91.8 & {\bf 12.6} & {\bf 92.0} & 254.6  & {\bf 92.0} & 367.1 & 90.8 & 489  \\
    \ijcnn & $2^{1}$ & $2^{1}$ & {\bf 98.8}  & {\bf 7.0} & {\bf 98.8} & 43.5  & {\bf 98.8}  & 111.6 & 95.4 & 227  \\
    \ijcnn & $2^{1}$ & $2^{6}$ & {\bf 98.3} & {\bf 34.6} & {\bf 98.3} & 584.5  & {\bf 98.3} & 1776.5 & 97.8 & 1085  \\
    \ijcnn & $2^{1}$ & $2^{10}$ & {\bf 97.2}  & {\bf 94.0} & {\bf 97.2}  & 523.1 & {\bf 97.2} & 1955.0 & 96.1 & 1691 \\
    \ijcnn & $2^{6}$ & $2^{-10}$ & {\bf 92.5}  & {\bf 27.8}  & 91.9 & 276.3  & 91.9  & 331.8 & 90.5  & 442  \\
    \ijcnn & $2^{6}$ & $2^{-6}$ & 94.8 & {\bf 19.9} & {\bf 95.6} & 313.7  & {\bf 95.6}  & 219.5 & 92.3 & 435 \\
    \ijcnn & $2^{6}$ & $2^{1}$ & {\bf 98.3}  & {\bf 6.4} & {\bf 98.3} & 75.3  & {\bf 98.3}  & 59.8 & 97.5 & 222  \\
    \ijcnn & $2^{6}$ & $2^{6}$ & {\bf 98.1} & {\bf 48.3} & {\bf 98.1}  & 384.5  & {\bf 98.1} & 987.7 & 97.1 & 1144 \\
    \ijcnn & $2^{6}$ & $2^{10}$ & {\bf 97.2}  & {\bf 51.9} & {\bf 97.2}  & 530.7  & {\bf 97.2}  & 1340.9 & 95.4 &1022 \\
    \ijcnn & $2^{10}$ & $2^{-10}$ & {\bf 94.4}  & {\bf 146.5}  & 92.5 & 606.1  & 92.5  & 1586.6 & 91.7 & 401  \\
    \ijcnn & $2^{10}$ & $2^{-6}$ & 97.3 & {\bf 124.3} & {\bf 97.6} & 553.6  & {\bf 97.6} & 1152.2 & 96.5 &1075  \\
    \ijcnn & $2^{10}$ & $2^{1}$ & 97.5 & {\bf 10.6} &  {\bf 97.5}& 50.8  & {\bf 97.5}  & 139.3 & 97.1 & 605 \\
    \ijcnn & $2^{10}$ & $2^{6}$ & {\bf 98.2} & {\bf 42.5} & {\bf 98.2} & 338.3  & {\bf 98.2}  & 1629.3 & 97.1 & 890  \\
    \ijcnn & $2^{10}$ & $2^{10}$ & {\bf 97.2}  & {\bf 66.4} & {\bf 97.2}  & 309.6  & {\bf 97.2} & 2398.3 & 95.4 & 909 \\
    \hline
  \end{tabular}
  }
  \vspace{-10pt}
\end{table*}

\begin{figure*}
  \begin{center}
\begin{tabular}{ccc}
  \vspace{-10pt}
  \subfloat[\ijcnn $C=2^{-10}$]{\includegraphics[width=0.28\textwidth]{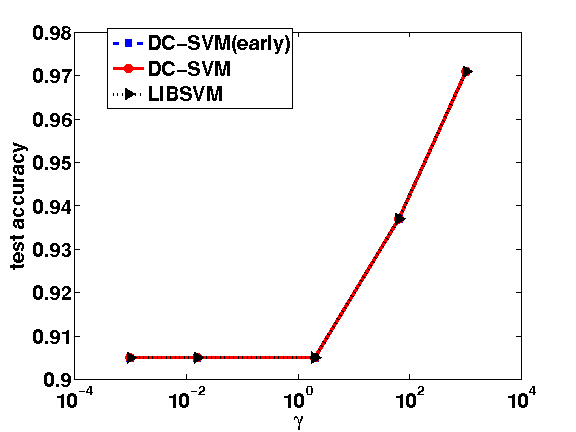}}
&
\subfloat[\ijcnn $C = 2^{1}$]{\includegraphics[width=0.28\textwidth]{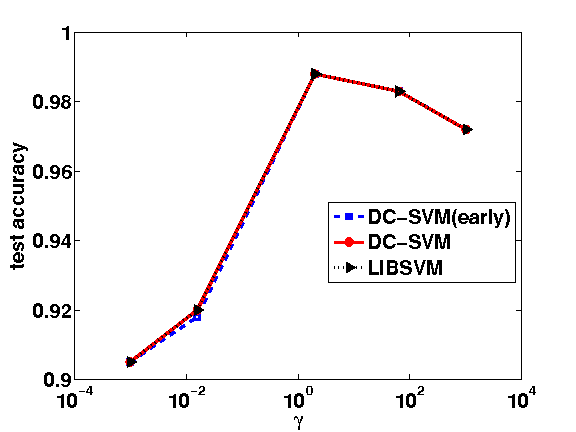}} & 
\subfloat[\ijcnn $C = 2^{10}$]{\includegraphics[width=0.28\textwidth]{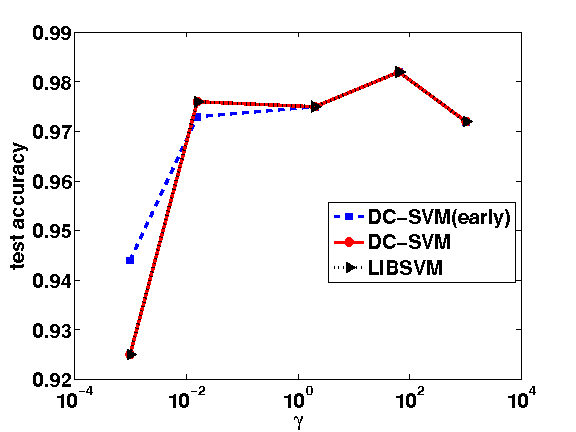}} \\\
\subfloat[\ijcnn $\gamma=2^{-10}$]{\includegraphics[width=0.28\textwidth]{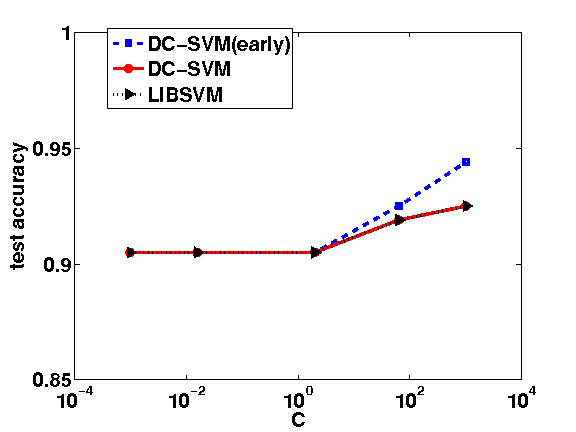}} &
\subfloat[\ijcnn $\gamma=2^{1}$]{\includegraphics[width=0.28\textwidth]{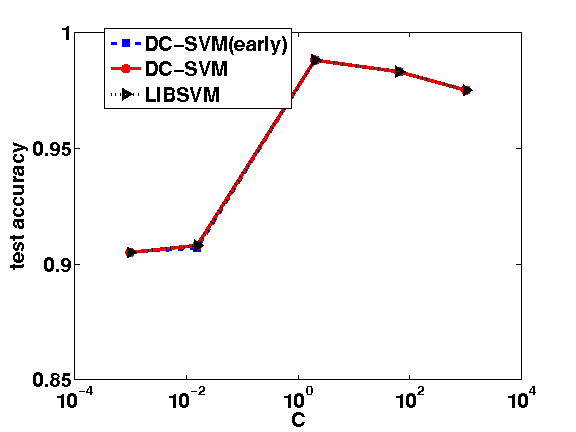}} &
\subfloat[\ijcnn $\gamma=2^{10}$]{\includegraphics[width=0.28\textwidth]{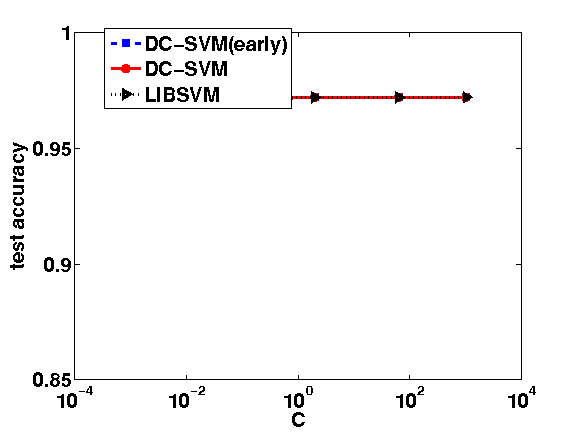}} 
\\
\end{tabular}
\end{center}
\caption{Robustness to the parameters $C, \gamma$ on \ijcnn dataset. \label{fig:parameter_plot_ijcnn}}
\vspace{-10pt}
\end{figure*}
\newpage
\begin{table*}[inner sep=0pt]
  \vspace{-20pt}
  \centering
  \caption{Comparison of DC-SVM, DC-SVM (early) and  LIBSVM on \webspam with various parameters $C, \gamma$. DC-SVM (early) is always more than 30 times faster than LIBSVM and has comparable or better test accuracy; DC-SVM is faster than LIBSVM under all settings.  }
  \vspace{-10pt}
  \label{tab:c_gamma_webspam}
\resizebox{12cm}{!}{
  \begin{tabular}{|c|r|r|rr|rr|rr|}
    \hline
    \multirow{2}{*}{dataset} & \multirow{2}{*}{$C$} & \multirow{2}{*}{$\gamma$} & \multicolumn{2}{|c}{DC-SVM (early)} & 
    \multicolumn{2}{|c}{DC-SVM}  & \multicolumn{2}{|c|}{\LIBSVM} \\ 
    \cline{4-9}
  & & & acc(\%) & time(s) & acc(\%) & time(s) & acc(\%) & time(s) \\
    \hline
    \webspam & $2^{-10}$ & $2^{-10}$ & {\bf 86} & {\bf 806}  & 61 &  26324 & 61  &  45984    \\
    \webspam & $2^{-10}$ & $2^{-6}$ & {\bf 83} & {\bf 935}   & 61 & 22569 &  61  &  53569    \\
    \webspam & $2^{-10}$ & $2^{1}$ & 87.1 & {\bf 886}  & {\bf 91.1} & 10835 & {\bf 91.1}   & 34226    \\
    \webspam & $2^{-10}$ & $2^{6}$ & {\bf 93.7} & {\bf 1060} &  92.6 & 6496  & 92.6  & 34558   \\
    \webspam & $2^{-10}$ & $2^{10}$ & 98.3  & {\bf 1898} & {\bf 98.5}  & 7410  & {\bf 98.5}  & 55574   \\
    \webspam & $2^{-6}$ & $2^{-10}$ & {\bf 83}  & {\bf 793}  & 68  & 24542  & 68  & 44153    \\
    \webspam  & $2^{-6}$ & $2^{-6}$ & {\bf 84} & {\bf 762} & 69  & 33498 & 69  & 63891    \\
    \webspam & $2^{-6}$ & $2^{1}$ & 93.3  & {\bf 599} & {\bf 93.5} & 15098 & 93.1 & 34226   \\
    \webspam  & $2^{-6}$ & $2^{6}$ & {\bf 96.4} & {\bf 704} & {\bf 96.4} & 7048 & {\bf 96.4}  & 48571   \\
    \webspam   & $2^{-6}$ & $2^{10}$ & 98.3  & {\bf 1277} & {\bf 98.6}  &6140  & {\bf 98.6} & 45122    \\
    \webspam  & $2^{1}$ & $2^{-10}$ & {\bf 87}  & {\bf 688}  & 78 & 18741   & 78  & 48512  \\
    \webspam  & $2^{1}$ & $2^{-6}$ & {\bf 93} & {\bf 645} & 81 & 10481  & 81 & 30106   \\
    \webspam  & $2^{1}$ & $2^{1}$ & 98.4  & {\bf 420} & {\bf 99.0} & 9157  & {\bf 99.0}  & 35151  \\
    \webspam  & $2^{1}$ & $2^{6}$ & {\bf 98.9} & {\bf 466} & {\bf 98.9} & 5104  & {\bf 98.9} & 28415  \\
    \webspam  & $2^{1}$ & $2^{10}$ & 98.3  & {\bf 853} & {\bf 98.7} & 4490  & 98.7 & 28891  \\
    \webspam   & $2^{6}$ & $2^{-10}$ & {\bf 93}  & {\bf 759}  & 80 & 24849  & 80  & 64121  \\
    \webspam  & $2^{6}$ & $2^{-6}$ & {\bf 97} & {\bf 602} & 83 & 21898  & 83  & 55414  \\
    \webspam  & $2^{6}$ & $2^{1}$ & 98.8  & {\bf 406} & {\bf 99.1} & 8051  & {\bf 99.1}  & 40510 \\
    \webspam  & $2^{6}$ & $2^{6}$ & {\bf 99.0} & {\bf 465} & 98.9  & 6140  & 98.9 & 35510  \\
    \webspam  & $2^{6}$ & $2^{10}$ & 98.3  & {\bf 917} & {\bf 98.7}  & 4510  & {\bf 98.7}  & 34121  \\
    \webspam  & $2^{10}$ & $2^{-10}$ & {\bf 97}  & {\bf 1350}  & 82 & 31387  & 82  & 81592  \\
    \webspam  & $2^{10}$ & $2^{-6}$ & {\bf 98} & {\bf 1127} & 86 & 34432  & 86 & 82581 \\
    \webspam  & $2^{10}$ & $2^{1}$ & {\bf 98.8} & {\bf 463} &  {\bf 98.8}& 10433  & {\bf 98.8}  & 58512 \\
    \webspam  & $2^{10}$ & $2^{6}$ & {\bf 99.0} & {\bf 455} & {\bf 99.0} & 15037  & {\bf 99.0}  & 75121 \\
    \webspam   & $2^{10}$ & $2^{10}$ & 98.3  & {\bf 831} & {\bf 98.7}  & 7150  & {\bf 98.7} & 59126 \\
    \hline
  \end{tabular}
  }
  \vspace{-10pt}
\end{table*}

\begin{figure*}
  \vspace{-10pt}
  \begin{center}
\begin{tabular}{ccc}
  \subfloat[\covtype $C=2^{-10}$]{\includegraphics[width=0.31\textwidth]{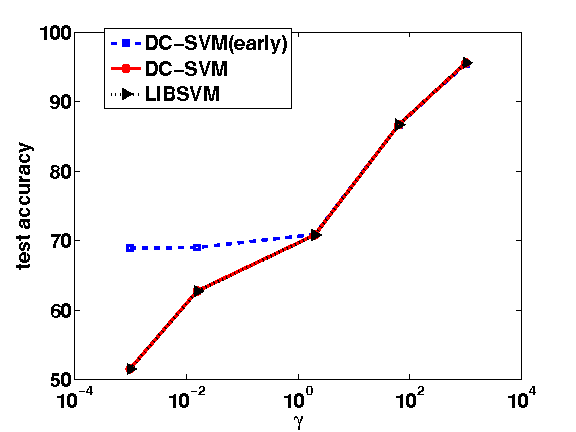}}
&
\subfloat[\covtype $C = 2^{1}$]{\includegraphics[width=0.31\textwidth]{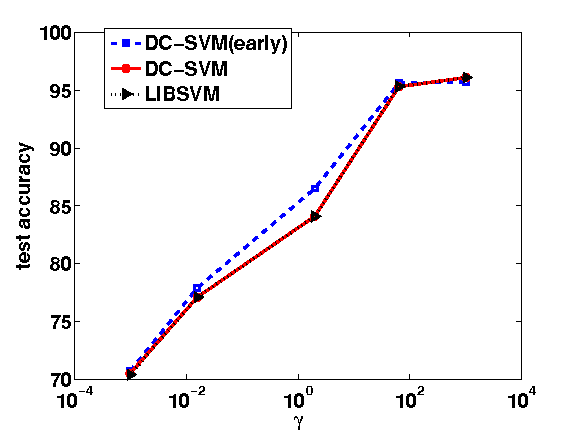}} & 
\subfloat[\covtype $C = 2^{10}$]{\includegraphics[width=0.31\textwidth]{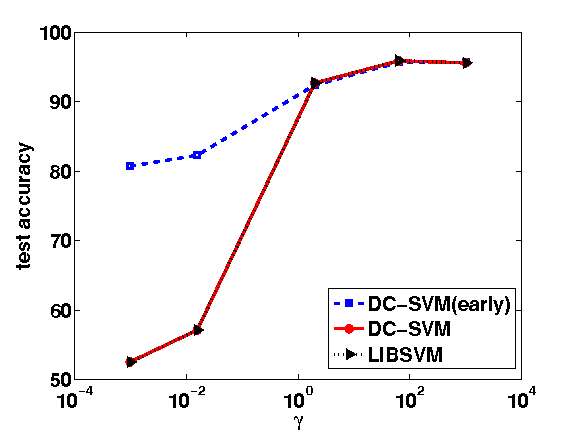}} \\
\subfloat[\covtype $\gamma=2^{-10}$]{\includegraphics[width=0.31\textwidth]{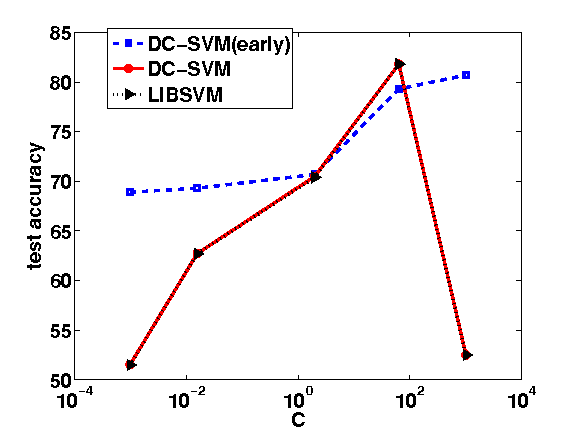}} &
\subfloat[\covtype $\gamma=2^{1}$]{\includegraphics[width=0.31\textwidth]{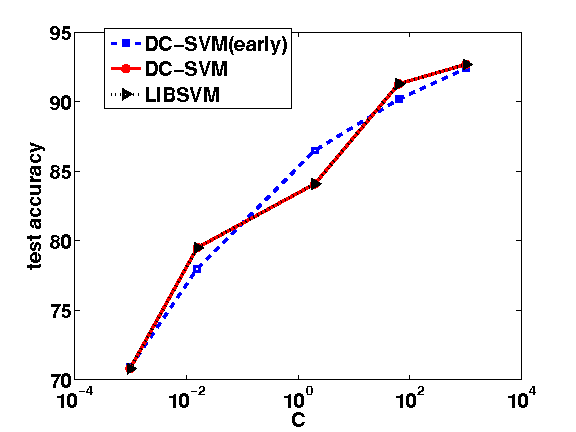}} &
\subfloat[\covtype $\gamma=2^{10}$]{\includegraphics[width=0.31\textwidth]{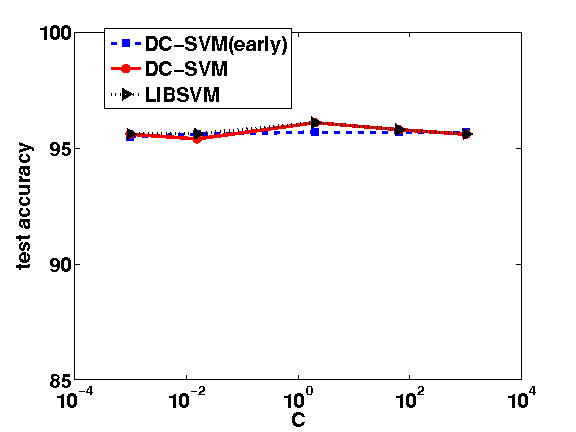}} 
\\
\end{tabular}
\end{center}
\caption{ Robustness to the parameters $C, \gamma$ on \covtype dataset. \label{fig:parameter_plot_covtype} }
\vspace{-10pt}
\end{figure*}
\newpage

\begin{table*}[inner sep=0pt]
  \vspace{-20pt}
  \centering
  \caption{Comparison of DC-SVM, DC-SVM (early) and LIBSVM on \covtype with various parameters $C, \gamma$. DC-SVM (early) is always more than 50 times faster than LIBSVM with similar test accuracy; DC-SVM is faster than LIBSVM under all settings.  }
  \label{tab:c_gamma_covtype}
  \vspace{-10pt}
\resizebox{12cm}{!}{
  \begin{tabular}{|c|r|r|r|r|r|r|r|r|r|r|}
    \hline
    \multirow{2}{*}{dataset} & \multirow{2}{*}{$C$} & \multirow{2}{*}{$\gamma$} & \multicolumn{2}{|c}{DC-SVM (early)} & 
    \multicolumn{2}{|c}{DC-SVM}  & \multicolumn{2}{|c|}{\LIBSVM} \\ 
    \cline{4-9}
  & & & acc(\%) & time(s) & acc(\%) & time(s) & acc(\%) & time(s) \\
    \hline
    \covtype & $2^{-10}$ & $2^{-10}$ & {\bf 68.9} & {\bf 736}  & 51.5 &  24791 & 51.5  & 48858     \\
    \covtype   & $2^{-10}$ & $2^{-6}$ & {\bf 69.0} & {\bf 507}   & 62.7 & 17189  &  62.7 & 62668     \\
    \covtype & $2^{-10}$ & $2^{1}$ & {\bf 70.9} & {\bf 624}  & 70.8 & 12997 & 70.8   & 88160    \\
    \covtype  & $2^{-10}$ & $2^{6}$ & {\bf 86.7} & {\bf 1351} &  {\bf 86.7} & 13985  & {\bf 86.7}  & 85111   \\
    \covtype & $2^{-10}$ & $2^{10}$ & 95.5 & {\bf 1173} & {\bf 95.6}  & 9480  & {\bf 95.6}  & 54282  \\
    \covtype  & $2^{-6}$ & $2^{-10}$ & {\bf 69.3}  & {\bf 373}  & 62.7  & 10387  & 62.7  & 90774    \\
    \covtype  & $2^{-6}$ & $2^{-6}$ & {\bf 70.0} & {\bf 625} & 68.6  & 14398 & 68.6  & 76508   \\
    \covtype  & $2^{-6}$ & $2^{1}$ & 78.0  & {\bf 346} & {\bf 79.5} & 5312 & {\bf 79.5} & 77591  \\
    \covtype  & $2^{-6}$ & $2^{6}$ & 87.9 & {\bf 895} & {\bf 87.9} & 8886 & {\bf 87.9}  & 120512  \\
    \covtype & $2^{-6}$ & $2^{10}$ & {\bf 95.6}  & {\bf 1238} & 95.4  & 7581  & 95.6 & 123396   \\
    \covtype  & $2^{1}$ & $2^{-10}$ & {\bf 70.7}  & {\bf 433}  & 70.4 & 25120   & 70.4  & 88725  \\
    \covtype  & $2^{1}$ & $2^{-6}$ & {\bf 77.9} & {\bf 1000} & 77.1 & 18452  & 77.1 & 69101   \\
    \covtype   & $2^{1}$ & $2^{1}$ & {\bf 86.5}  & {\bf 421} & 84.1 & 11411  & 84.1  & 50890  \\
    \covtype  & $2^{1}$ & $2^{6}$ & {\bf 95.6} & {\bf 299} & 95.3 & 8714  & 95.3 & 117123   \\
    \covtype   & $2^{1}$ & $2^{10}$ & 95.7  & {\bf 882} & {\bf 96.1} & 5349  & & $>$300000   \\
    \covtype & $2^{6}$ & $2^{-10}$ & 79.3  & {\bf 1360}  & {\bf 81.8} & 34181  & 81.8  & 105855   \\
    \covtype  & $2^{6}$ & $2^{-6}$ & 81.3 & {\bf 2314} & {\bf 84.3} & 24191  & {\bf 84.3}  & 108552   \\
    \covtype   & $2^{6}$ & $2^{1}$ & 90.2  & {\bf 957} & {\bf 91.3} & 14099  & {\bf 91.3}  & 75596  \\
    \covtype   & $2^{6}$ & $2^{6}$ & {\bf 96.3} & {\bf 356} & 96.2  & 9510  & 96.2 & 92951   \\
    \covtype  & $2^{6}$ & $2^{10}$ & 95.7  & {\bf 961} & {\bf 95.8}  & 7483  & {\bf 95.8}  & 288567   \\
    \covtype  & $2^{10}$ & $2^{-10}$ & {\bf 80.7}  & {\bf 5979}  & 52.5 & 50149  & 52.5  & 235183   \\
    \covtype   & $2^{10}$ & $2^{-6}$ & {\bf 82.3} & {\bf 8306} & 57.1 & 43488  & & $>$ 300000  \\
    \covtype   & $2^{10}$ & $2^{1}$ & 92.4 & {\bf 4553} &  {\bf 92.7} & 19481  & {\bf 92.7}  & 254130  \\
    \covtype   & $2^{10}$ & $2^{6}$ & 95.7 & {\bf 368} & {\bf 95.9} & 12615  & {\bf 95.9}  & 93231  \\
    \covtype  & $2^{10}$ & $2^{10}$ & {\bf 95.7}  & {\bf 1094} & 95.6  & 10432  & 95.6 & 169918  \\
    \hline
  \end{tabular}
  }
  \vspace{-10pt}
\end{table*}

\begin{figure*}
  \vspace{-10pt}
  \begin{center}
\begin{tabular}{ccc}
  \subfloat[\webspam $C=2^{-10}$]{\includegraphics[width=0.31\textwidth]{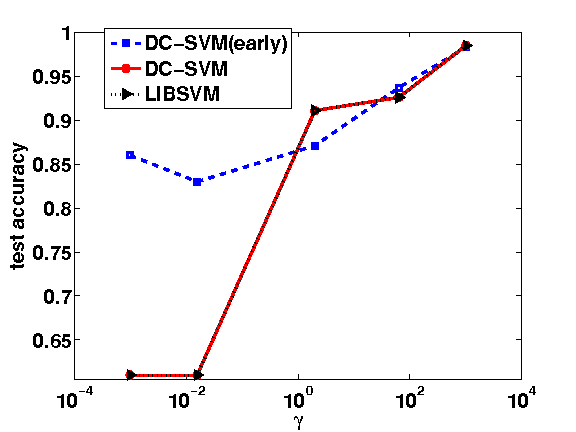}}
&
\subfloat[\webspam $C = 2^{1}$]{\includegraphics[width=0.31\textwidth]{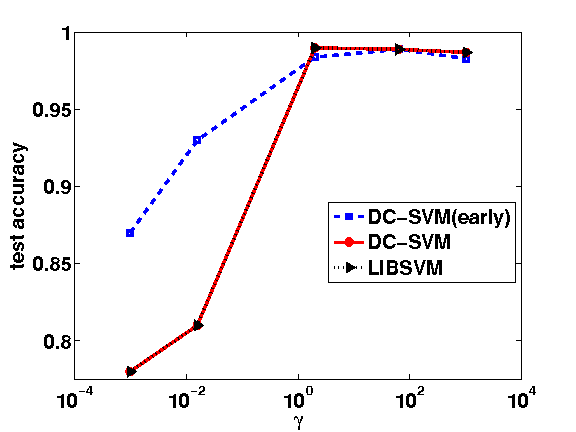}} & 
\subfloat[\webspam $C = 2^{10}$]{\includegraphics[width=0.31\textwidth]{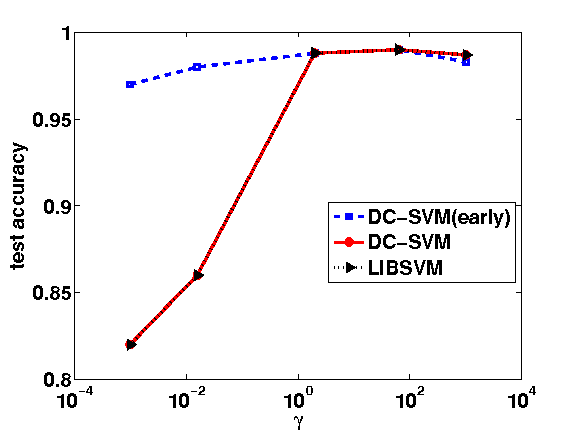}} \\
\subfloat[\webspam $\gamma=2^{-10}$]{\includegraphics[width=0.31\textwidth]{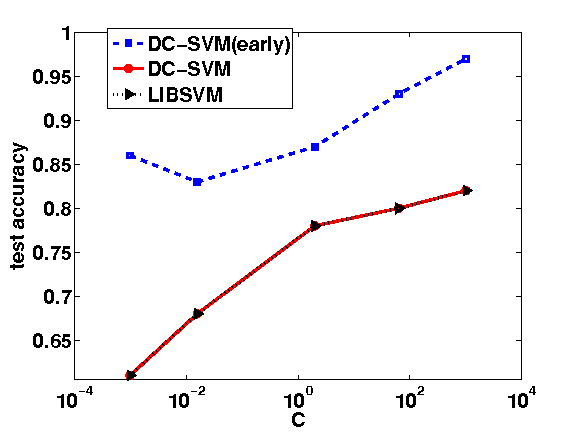}} &
\subfloat[\webspam $\gamma=2^{1}$]{\includegraphics[width=0.31\textwidth]{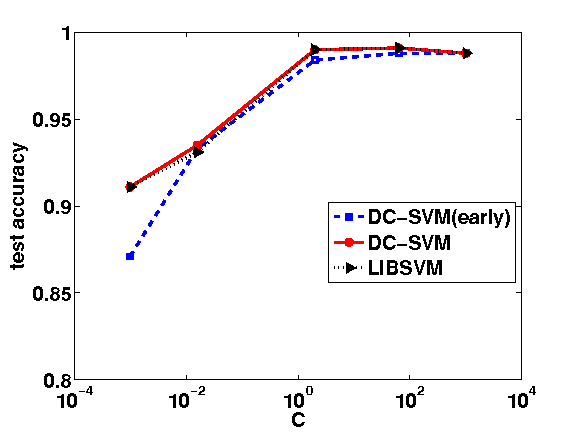}} &
\subfloat[\webspam $\gamma=2^{10}$]{\includegraphics[width=0.31\textwidth]{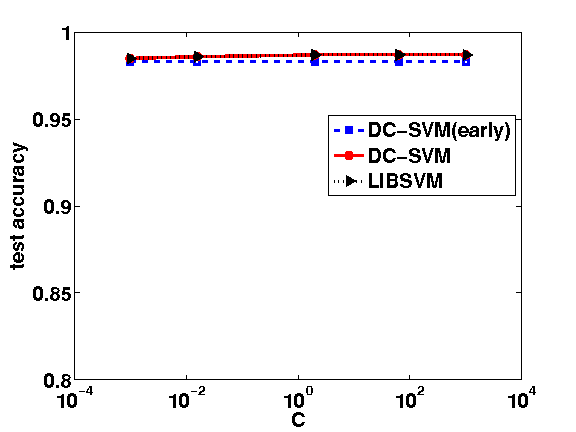}} 
\\
\end{tabular}
\end{center}
\caption{Robustness to the parameters $C, \gamma$ on \webspam dataset. \label{fig:parameter_plot_webspam} }
\vspace{-10pt}
\end{figure*}
\newpage

\begin{table*}[inner sep=0pt]
  \vspace{-20pt}
  \centering
  \caption{Comparison of DC-SVM, DC-SVM (early) and LIBSVM on \census with various parameters $C, \gamma$. DC-SVM (early) is always more than 50 times faster than LIBSVM with similar test accuracy; DC-SVM is faster than LIBSVM under all settings.  }
  \label{tab:c_gamma_census}
  \vspace{-10pt}
\resizebox{12cm}{!}{
  \begin{tabular}{|c|r|r|r|r|r|r|r|r|r|r|}
    \hline
    \multirow{2}{*}{dataset} & \multirow{2}{*}{$C$} & \multirow{2}{*}{$\gamma$} & \multicolumn{2}{|c}{DC-SVM (early)} & 
    \multicolumn{2}{|c}{DC-SVM}  & \multicolumn{2}{|c|}{\LIBSVM} \\ 
    \cline{4-9}
  & & & acc(\%) & time(s) & acc(\%) & time(s) & acc(\%) & time(s) \\
    \hline
    \census & $2^{-10}$ & $2^{-10}$   & {\bf 93.80} & {\bf 161}   & {\bf 93.80} & 2153  &  {\bf 93.80} & 3061  \\
    \census   & $2^{-10}$ & $2^{-6}$   & {\bf 93.80} & {\bf 166}  & {\bf 93.80} & 3316    &   {\bf 93.80}  & 5357 \\
    \census & $2^{-10}$ & $2^{1}$   & 93.61  & {\bf 202} &   {\bf 93.68} & 4215  &   93.66 & 11947 \\
    \census & $2^{-10}$ & $2^{6}$  & 91.96 & {\bf 228}   & {\bf 92.08} & 5104   & {\bf  92.08} & 12693 \\
    \census & $2^{-10}$ & $2^{10}$  & {\bf 62.00} & {\bf 195}   & 56.32 & 4951   &  56.31 & 13604\\
    \census  & $2^{-6}$ & $2^{-10}$   & {\bf 93.80} & {\bf 145}   & {\bf 93.80} & 3912   &   {\bf 93.80} & 6693 \\
    \census  & $2^{-6}$ & $2^{-6}$  & {\bf 93.80} &{\bf 149}   & {\bf 93.80} & 3951  &   {\bf 93.80} & 6568 \\
    \census  & $2^{-6}$ & $2^{1}$   & 93.63 & {\bf 217}  & {\bf 93.66} & 4145   &  {\bf 93.66} & 11945\\
    \census  & $2^{-6}$ & $2^{6}$   & 91.97 & {\bf 230}   & {\bf 92.10} & 4080   & {\bf 92.10} & 9404 \\
    \census & $2^{-6}$ & $2^{10}$   & {\bf 62.58} & {\bf 189}   & 56.32 & 3069   & 56.31 & 9078  \\
    \census  & $2^{1}$ & $2^{-10}$   & 93.80 &{\bf 148}   & {\bf 93.95} & 2057    &  {\bf 93.95} & 1908 \\
    \census  & $2^{1}$ & $2^{-6}$   & 94.55 &{\bf 139}    & {\bf 94.82} & 2018   &  {\bf 94.82} & 1998 \\
    \census  & $2^{1}$ & $2^{1}$   & 93.27 & {\bf 179}  & {\bf 93.36} & 4031    & {\bf 93.36 } & 37023\\
    \census  & $2^{1}$ & $2^{6}$   & 91.96 & {\bf 220}   & {\bf 92.06} & 6148   &  {\bf 92.06} & 33058 \\
    \census  & $2^{1}$ & $2^{10}$   & {\bf 62.78} &{\bf 184}   & 56.31 & 6541   &  56.31 & 35031 \\
    \census & $2^{6}$ & $2^{-10}$   & {\bf 94.66} &{\bf 193}   & {\bf 94.66}  & 3712  &  {\bf 94.69} & 3712 \\
    \census  & $2^{6}$ & $2^{-6}$  & 94.76 & {\bf 164}   & {\bf 95.21} & 2015   &  {\bf 95.21} & 3725 \\
    \census  & $2^{6}$ & $2^{1}$   & 93.10 & {\bf 229}  & {\bf 93.15} & 6814   & {\bf 93.15} & 32993 \\
    \census  & $2^{6}$ & $2^{6}$  & 91.77 & {\bf 243}  & {\bf 91.88} & 9158  &  {\bf 91.88} & 34035 \\
    \census  & $2^{6}$ & $2^{10}$   & {\bf 62.18} & {\bf 210}   & 56.25& 9514   &  56.25  & 36910 \\
    \census  & $2^{10}$ & $2^{-10}$   & 94.83 & {\bf 538}  &  94.83 & 2751   & {\bf 94.85}  & 8729 \\
    \census  & $2^{10}$ & $2^{-6}$ & {\bf 93.89} & {\bf 315} & 92.94& 3548   & 92.94 & 12735 \\
    \census  & $2^{10}$ & $2^{1}$  & 92.89 & {\bf 342} &  92.92 & 9105  & {\bf 92.93} & 52441 \\
    \census  & $2^{10}$ & $2^{6}$  & 91.64 & {\bf 244} &  {\bf 91.81} & 7519    & {\bf 91.81} & 34350 \\
    \census & $2^{10}$ & $2^{10}$  & {\bf 61.14} & {\bf 206} &   56.25 & 5917   & 56.23 & 34906 \\
    \hline
  \end{tabular}
  }
  \vspace{-10pt}
\end{table*}

\begin{figure*}
  \vspace{-10pt}
  \begin{center}
\begin{tabular}{ccc}
  \subfloat[\census $C=2^{-10}$]{\includegraphics[width=0.31\textwidth]{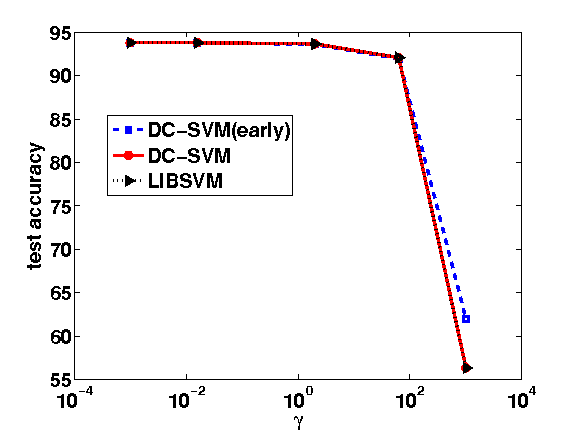}}
&
\subfloat[\census $C = 2^{1}$]{\includegraphics[width=0.31\textwidth]{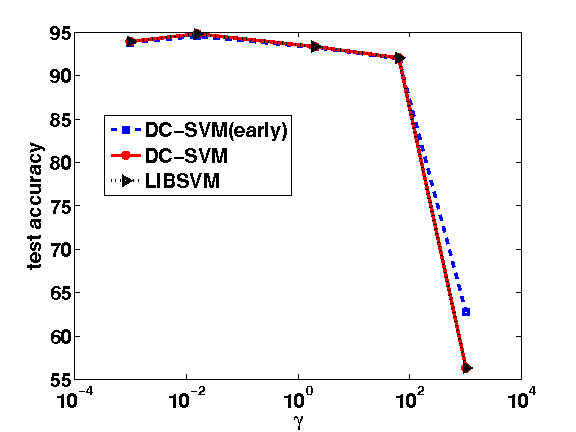}} & 
\subfloat[\census $C = 2^{10}$]{\includegraphics[width=0.31\textwidth]{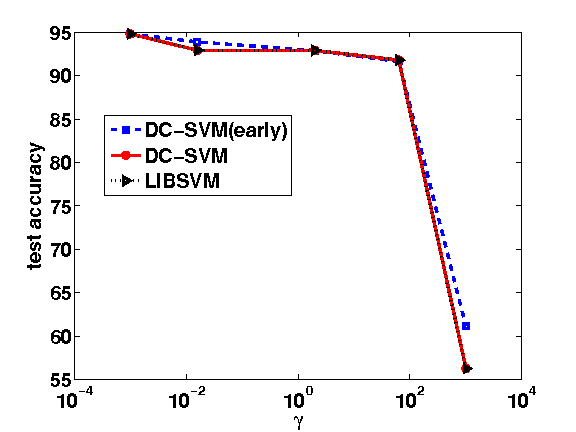}} \\
\subfloat[\census $\gamma=2^{-10}$]{\includegraphics[width=0.31\textwidth]{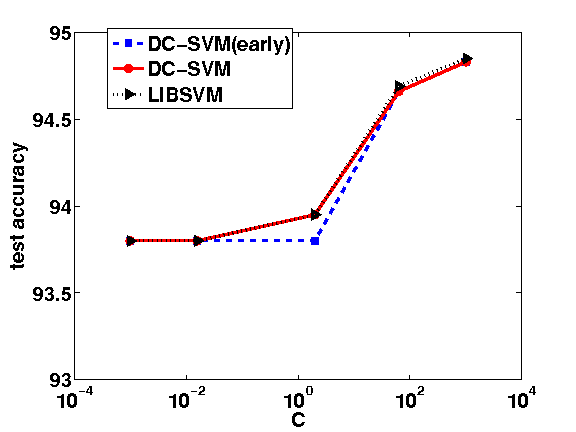}} &
\subfloat[\census $\gamma=2^{1}$]{\includegraphics[width=0.31\textwidth]{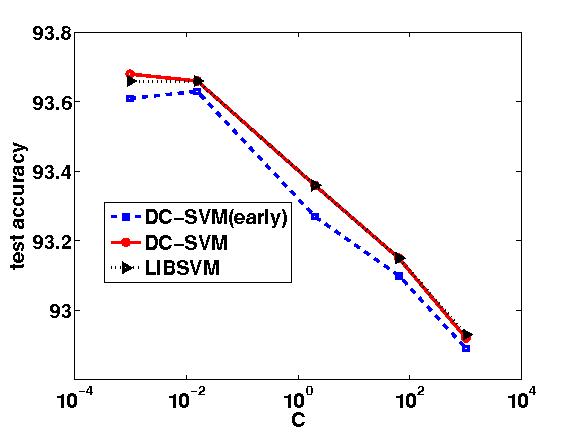}} &
\subfloat[\census $\gamma=2^{10}$]{\includegraphics[width=0.31\textwidth]{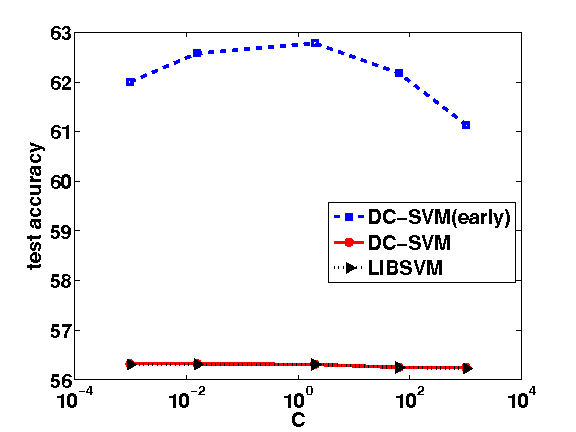}} 
\\
\end{tabular}
\end{center}
\caption{ Robustness to the parameters $C, \gamma$ on \census dataset. \label{fig:parameter_plot_census}}
\vspace{-10pt}
\end{figure*}

\newpage
%\subsection{Additional figures}

\end{document}